\documentclass[sigconf,nonacm]{acmart}

\AtBeginDocument{}

\setcopyright{none}          \acmYear{2027}
\acmConference[KDD 2027]{ACM KDD 2027}{August 1--5, 2027}{San Jose, United States}
\usepackage{graphicx}
\usepackage{booktabs}
\usepackage{amsmath}

\usepackage{amssymb}
\usepackage{multirow}
\usepackage{algorithm}
\usepackage{enumitem}
\usepackage{algpseudocode}
\usepackage[table]{xcolor}

\begin{document}

\title{Unequal Trips, Unequal Places: Diagnosing and Mitigating Delay Inequity in Autonomous Vehicle Fleet Coordination}

\author{Nicole Hu}
\affiliation{\department{School of Hotel and Tourism Management}
	\institution{The Hong Kong Polytechnic University}
	\city{Hong Kong SAR}
	\country{China}
}
\email{hulan.hu@connect.polyu.hk}

\author{Mingtao Zhang}
\affiliation{\department{Department of Computing}
	\institution{The Hong Kong Polytechnic University}
	\city{Hong Kong SAR}
	\country{China}
}
\email{mingtao.zhang@connect.polyu.hk}

\author{Haoyang Li}
\affiliation{\department{Department of Computing}
	\institution{The Hong Kong Polytechnic University}
	\city{Hong Kong SAR}
	\country{China}
}
\email{haoyang-comp.li@polyu.edu.hk}

\author{Chen Jason Zhang}
\affiliation{\department{School of Hotel and Tourism Management \& Computing}
	\institution{The Hong Kong Polytechnic University}
	\city{Hong Kong SAR}
	\country{China}
}
\email{jason-c.zhang@polyu.edu.hk}

\author{Qing Li}
\affiliation{\department{Department of Computing}
	\institution{The Hong Kong Polytechnic University}
	\city{Hong Kong SAR}
	\country{China}
}
\email{csqli@comp.polyu.edu.hk}

\begin{abstract}
City-scale autonomous vehicle fleet coordinators are typically optimized for aggregate travel time, yet fleet averages conceal how delay is distributed across trips and regions. We conduct a distributional audit on three real-city road-network and taxi-demand datasets from Manhattan, Chicago, and San Francisco. The audit reveals pervasive trip-length inequity whose direction depends on the city and coordinator. After accounting for trip length, spatial inequity becomes more pronounced as demand grows and is consistently stronger when trips are grouped by origin rather than destination. These findings motivate SPatially Aware RErouting (SPARE), a budgeted online coordination framework that assigns limited replanning capacity to delayed vehicles and redirects them using recently observed waiting pressure. SPARE provides a per-review decision guarantee and explicitly bounds online route updates. Experiments on all three datasets against six representative baselines show that SPARE delivers the strongest joint efficiency–fairness performance while retaining city-scale scalability. The results demonstrate that bounded congestion-responsive rerouting improves performance and equity without full-fleet replanning. 

\end{abstract}


\maketitle

{

\section{Introduction}

Urban transportation systems suffer from a basic efficiency paradox. Routing decisions that are optimal for each vehicle can combine into system-wide congestion that slows every participant~\cite{poa,congestiongames}. Autonomous vehicles (AVs) present a transformative opportunity, as coordination algorithms could adjust vehicle behavior to traffic conditions in real time, organizing circulating flows at congested intersections while permitting direct routes when roads are clear~\cite{li2026local}. This gives rise to the task of city-scale fleet coordination, where thousands of AVs are routed over one shared road network under shared capacity constraints. AV fleets already carry passengers in everyday robotaxi service~\cite{waymosafety}, coordinated fleets are expected to extend to delivery, logistics, and city-wide mobility services~\cite{fleetmanage,rethinkdispatch}, and how a city's delay is produced and distributed has long been a central concern of urban computing~\cite{urbancomputing}.

Depending on the underlying technique, existing coordination approaches fall into two categories. First, classical traffic assignment treats routing as a flow problem, assigning demand to a user equilibrium or a system optimum of the network~\cite{wardrop,sheffi}, and later variants interpolate between the two or bound how far any driver is detoured~\cite{itap,constrainedso}. Second, multi-agent path finding (MAPF) treats each vehicle as an agent, planning a path per vehicle and resolving conflicts at shared nodes and edges~\cite{mapfsurvey}. Methods in this line scale from priority-based rules over thousands of agents~\cite{pibt}, to congestion-aware replanning at every step~\cite{gpibt}, and most recently to real city road graphs, where global shortest-path guidance is combined with local conflict resolution~\cite{li2026local}. This latest generation coordinates thousands of vehicles on real networks and is the setting we study. Despite their methodological differences, both lines are designed and evaluated primarily through aggregate efficiency, namely fleet-level travel delay or throughput.

Optimizing aggregate efficiency alone leaves a distributional fairness
problem~\cite{biasfairsurvey}. Average delay describes the fleet as a whole
but not how that delay is divided among trips. Two coordinators with the same
average can produce very different service: one may delay every trip mildly,
whereas another may impose severe delay on a smaller group. Minimizing total
travel time does not distinguish between them because it places no constraint
on who absorbs the delay. In a mobility service, however, trip delay is part
of the quality of service received by a passenger. Systematic differences in
slowdown across otherwise comparable trips are therefore operational
inequities even when no demographic attributes are available.

We audit representative methods from both coordination categories on three
real-city road networks under matched taxi demand. The audit separates two
dimensions of inequity. First, slowdown differs systematically across free-flow
trip-length groups, but the direction of the disparity changes across cities
and coordinators: shorter, longer, or intermediate trips may receive worse
proportional service. The recurring failure is therefore not that one length
group always loses, but that proportional delay is distributed unevenly
across length groups. Second, after conditioning on trip length, substantial
geographic disparities remain. They grow with network-relative demand under
both origin and destination grouping and recur across coordinators and cities.
Origin grouping consistently reveals the larger gap, suggesting that where a
trip begins is particularly consequential without reducing spatial inequity
to an exclusively origin-based phenomenon.

The diagnosis exposes a distinct online coordination problem. Static guide
paths cannot react to waiting pressure created during execution, yet
replanning every vehicle at every step is computationally prohibitive. We
therefore introduce \textbf{SP}atially \textbf{A}ware
\textbf{RE}routing (SPARE), a budgeted online fleet-coordination framework.
At each review, SPARE allocates a limited replanning budget to the vehicles
with the greatest accumulated delay and computes new guide paths that trade
free-flow distance against recently observed waiting pressure. The review
interval $K$ and reroute budget $R$ explicitly control its online cost. This
formulation couples delay-prioritized intervention with
congestion-responsive routing rather than repeatedly solving a full-fleet
planning problem. Evaluation on three real-city datasets against six
representative baselines shows that SPARE improves fleet-level efficiency,
reduces spatial disparity, and delivers the strongest joint
efficiency--fairness performance while retaining city-scale scalability.
The results establish that targeted online adaptation improves aggregate and
distributional outcomes together without repeated full-fleet replanning.

The contributions of this paper are summarized as follows.

\begin{itemize}[leftmargin=*]
	\item We conduct cross-city evaluations using official taxi records and uniformly processed road networks, and audit fleet-coordination delay across trip-length groups and geographic regions.
	\item We show that trip-length inequity is pervasive but direction-dependent
	across cities and coordinators. After controlling for these length effects,
	spatial inequity grows with network-relative demand under both origin and
	destination grouping, with origin grouping consistently revealing the
	larger disparity.
	\item We formulate budgeted online route adaptation and propose SPARE, which
	combines delay-prioritized vehicle selection with congestion-responsive
	path planning. SPARE provides a per-review decision guarantee and a bounded
	online planning budget; evaluation against six baselines shows
	that it improves efficiency and spatial fairness while retaining scalability.
\end{itemize}

}
  
 {
\section{Problem Setting and Related Work}
\label{sec:related}

\subsection{Fleet Coordination Model}
\label{sec:model}

Vehicles move on a directed road graph $G=(V,E)$ in discrete time. Each
node has an occupancy capacity, and each edge has a capacity derived from
its lane count. At every step, an active vehicle either advances along one
edge, which takes one time unit, or waits. Given origin--destination pairs
$(s_i,g_i)$ for vehicles $i=1{\ldots}N$, let $p_i$ be vehicle $i$'s current
position, $g_i$ its goal, and $P_i$ its remaining guide path. A
capacity-feasible movement resolver
\[
\textsc{ResolveStep}\bigl(G,\{p_i\},\{g_i\},\{P_i\}\bigr)
\]
returns the next positions and maintains valid remaining guide paths while
respecting node and edge capacities. This separation distinguishes
\emph{route adaptation}, which determines the guide paths, from
\emph{movement resolution}, which decides which feasible moves occur at the
current step.

Let $t_i$ denote the realized travel time of vehicle $i$. The conventional
coordination objective is to bring all vehicles to their destinations while
minimizing aggregate travel time $\sum_i t_i$. We retain the same operational
model but ask two distributional questions: which trips absorb the
coordination delay, and how can guide paths respond to the congestion formed
during execution without invoking full-fleet replanning?

\subsection{City-Scale Fleet Coordination}

Classical traffic assignment treats routing as a flow problem and
characterizes user-equilibrium or system-optimal assignments~\cite{wardrop,sheffi}.
Later methods interpolate between them~\cite{itap} or constrain individual
detours~\cite{constrainedso}. Learned path representations provide scalable
encodings for route-level tasks~\cite{lightpath}, whereas multi-agent path
finding represents each vehicle explicitly and resolves conflicts at shared
nodes and edges~\cite{mapfsurvey}. This line ranges from priority-based rules that
scale to thousands of agents~\cite{pibt}, through congestion-aware
replanning~\cite{gpibt}, to coordinators that combine global guidance with
local conflict resolution on real city graphs~\cite{li2026local}. Parallel
work on traffic forecasting has developed scalable spatiotemporal models for
large urban networks~\cite{patchstg}.

Related urban-computing systems address other fleet decisions, including
dispatch and repositioning~\cite{urbancomputing,fleetmanage,orderdispatch,rethinkdispatch},
network-wide signal control~\cite{presslight,bigdatatrafficsignals}, and
aggregate traffic shaping with a controlled fraction of
vehicles~\cite{mixedautonomy}. Recent urban-intelligence perspectives further
connect heterogeneous city observations to integrated governance
decisions~\cite{neuralcity}. Across these settings, the dominant criteria
remain total delay and throughput. 
Full congestion-aware replanning updates routes online but wastes effort replanning the entire fleet. SPARE instead frames online route adaptation as a budget-allocation problem: each review allocates limited replanning capacity to selected delayed vehicles, redirecting their paths based on currently observed congestion.

\subsection{Fairness in Mobility Systems}

Fairness has been studied in several adjacent stages of mobility, including
learning methods that explicitly reduce performance disparities across
locations~\cite{metaref}. Centralized routing can bound how far a driver is
detoured~\cite{constrainedso,fairrouting},
and interpolated assignment measures disparities among drivers sharing an
origin and destination~\cite{itap}. Mobility platforms balance earnings
across ride-hailing drivers~\cite{fairassign,ridefair}, correct demographic
bias in demand prediction~\cite{fairdemand}, reduce sample-level performance
heterogeneity in spatiotemporal forecasting~\cite{fairstg}, and measure
neighborhood differences in matching and pickup
service~\cite{ridehailineq}. These problems concern assignment, prediction,
or service access. We study a different stage: delay produced during movement
after a fleet has already been dispatched onto a shared, capacity-limited
network.

Our evaluation follows the broader principle that comparisons should
condition on task-relevant attributes~\cite{fairaware,eqopp,biasfairsurvey}.
Multi-agent work has likewise learned or constrained fair objectives
~\cite{marlfair,multiagentfair}, while fairness-aware MAPF has considered
envy-freeness and max-min welfare on synthetic grids~\cite{fairmapf}. Our
units are trips with different travel requirements rather than
interchangeable agents. Rather than presuming that one trip-length group is
systematically disadvantaged, we audit the complete slowdown profile and
then condition on free-flow trip length when comparing geographic groups.
This separates length-dependent service differences from systematic
differences associated with where a trip begins or ends.
}
 \section{Diagnosing Delay Inequity}
\label{sec:setup}
This section diagnoses two forms of delay disparity on three large-scale city
road networks. Section~\ref{sec:length} examines how slowdown varies with trip
length and shows that the disadvantaged length group depends on the city and
coordinator.
Section~\ref{sec:spatial} then conditions on trip length and compares regional
disparity by trip origin and destination. Spatial inequity grows with demand
under both groupings, while origin grouping consistently reveals the larger
gap. Together, the analyses identify who is delayed and where disparities
emerge.

\subsection{Data, Methods, and Metrics}
\label{sec:prelim}
The following data, protocol, baselines, and metrics govern both the
diagnostic analyses in this section and the comparative evaluation in
Section~\ref{sec:evaluation}.

\paragraph{Data and networks.}
We study three widely used real city networks, each paired with the official trip
records released by its own city and summarized in Table~\ref{tab:cities}.
Each drivable network is built from
OpenStreetMap by one city-agnostic pipeline, with node and edge capacities read from
recorded lane counts.
\begin{itemize}[leftmargin=*]
	\item \textbf{Manhattan}~\cite{nyctlc}, $4{,}594$ nodes and $9{,}856$ edges, with
	demand from the NYC TLC Trip Record Data, the standard public source
	for city-scale fleet studies.
	\item \textbf{Chicago}~\cite{chicagotaxi}, $18{,}702$ nodes and $48{,}331$ edges, with demand from the
City of Chicago Taxi Trips open data.
\item \textbf{San Francisco}~\cite{sfmtataxi}, $10{,}147$ nodes and $27{,}968$ edges, with demand
from the SFMTA taxi trip records.
\end{itemize}

\paragraph{OD construction and initialization.}
{Before OD construction, we remove trip records with invalid
coordinates, nonpositive duration, or endpoints outside the study network.}
Origin--destination pairs are sampled from the retained records rather than
drawn uniformly, preserving the spatial imbalance of actual travel. Endpoints
are matched to the nearest road node. All trips are released at $t=0$ by
default as a controlled stress test.

\paragraph{Methods.}
Across the paper, we consider six representative baselines spanning routing,
traffic assignment, multi-agent coordination, and explicit fairness.
\begin{itemize}[leftmargin=*]
	\item \textbf{SP}~\cite{poa,congestiongames} and \textbf{GSP}: selfish
	shortest-path routing with contention delays, and a greedy variant that
	selects the feasible neighboring node with the shortest remaining
	free-flow distance.
	\item \textbf{TAP}~\cite{wardrop,sheffi}: a user-equilibrium traffic-assignment
	baseline.
	\item \textbf{PIBT}~\cite{pibt}: priority-based multi-agent path finding.
	\item \textbf{FAIR-RD}: a scheduling baseline motivated by max-min
	fairness over agents~\cite{fairmapf}; when vehicles contend, it gives
	priority to trips that are already running behind.
	\item \textbf{GLC}~\cite{li2026local}: a coordinator that combines global
	shortest-path guidance with local priority inheritance and backtracking.
\end{itemize}
The diagnostic analyses use the scalable subset required for each city and
demand sweep, with the included methods stated in the corresponding figure.
Section~\ref{sec:density} evaluates SPARE against the complete baseline suite.
TAP is omitted from the large-demand diagnostic sweeps because it does not
complete within the runtime limit.

\paragraph{Evaluation protocol.}
We normalize demand by network size as $\rho=N/|V|$, with the value of
$\rho$ reported in each figure. Results are averaged over ten matched OD
samples, with 95\% confidence intervals unless otherwise stated.

\paragraph{Metrics used throughout the paper.}
For a trip $i$, let $t_i$ be its realized travel time and $\ell_i$ its
free-flow travel time, both measured in steps. Free-flow time is the travel
time of the shortest path between the trip's origin and destination in an empty
network, so $t_i \ge \ell_i$.

\begin{itemize}[leftmargin=*]
	\item \textbf{Slowdown.}
	$\sigma_i=t_i/\ell_i \ge 1$ measures how much longer trip $i$ takes than
	its free-flow shortest path.
	
	\item \textbf{Overhead (OH).}
	$\mathrm{OH}=\frac{\sum_i t_i}{\sum_i \ell_i}-1 \ge 0$ is the fleet's
	aggregate travel-time overhead relative to the free-flow lower bound.
	
	\item \textbf{Length-group disparity (LGD).}
	We partition trips into $B=10$ equal-size bins by free-flow time. Let
	$I_k$ contain the trips in bin $k$, let $b(i)$ be the bin of trip $i$, and let
	$\overline{\sigma}_k=|I_k|^{-1}\sum_{i\in I_k}\sigma_i$ be the mean slowdown
	of bin $k$. Section~\ref{sec:length} reports the complete decile profile,
	while $\mathrm{LGD}=\mathrm{std}_{k=1,\ldots,B}
	(\overline{\sigma}_k)$ summarizes its variation. A larger LGD indicates
	greater disparity across trip-length groups.

	\item \textbf{Spatial inequity (SI).}
	Within each experimental condition, we remove the length profile using
	$r_i=\sigma_i-\overline{\sigma}_{b(i)}$. For an origin or destination taxi
	zone $g$, let $\bar r_g$ be the mean residual slowdown of its trips. We define
	\begin{equation}
		\mathrm{SI}
		= \mathrm{std}_{g\in\mathcal{G}^{+}}(\bar r_g),
		\label{eq:si}
	\end{equation}
	where $\mathcal{G}^{+}$ contains the represented zones and $\mathrm{std}$
	is the unweighted population standard deviation. Larger SI indicates greater
	between-zone disparity; origin and destination grouping give
	$\mathrm{SI}_{\mathrm{o}}$ and $\mathrm{SI}_{\mathrm{d}}$, respectively.
	Because taxi-zone systems differ, we compare SI within each city.
\end{itemize}
\begin{table}[t]
	\centering
	\small
	\vspace{-1em}
\caption{Road-network and demand statistics.}
	\vspace{-1em}
	\label{tab:cities}
	\begin{tabular}{@{}lrrl@{}}
		\hline
		City & Nodes & Edges & Demand source \\
		\hline
		Manhattan & $4{,}594$ & $9{,}856$ & NYC TLC records \\
		Chicago & $18{,}702$ & $48{,}331$ & Chicago taxi trips \\
		San Francisco & $10{,}147$ & $27{,}968$ & SFMTA taxi trips \\
		\hline
	\end{tabular}
	\vspace{-1em}
\end{table}
\subsection{Trip-Length Inequity Varies}
\label{sec:length}

\begin{figure}[t]
	\centering
	\includegraphics[width=\linewidth]{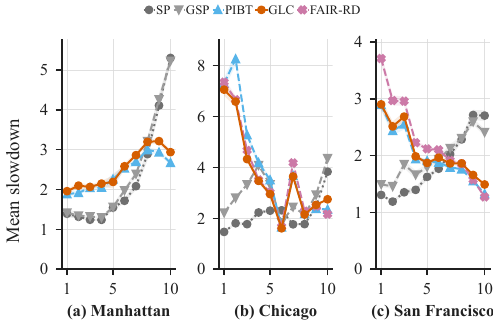}
		\vspace{-1em}
		\caption{Trip-length inequity   across cities.}
				\vspace{-2em}
	\label{fig:length}
\end{figure}

Slowdown separates proportional delay from absolute trip duration. A flat
decile profile represents parity across trip lengths, whereas any systematic
rise, decline, or non-monotonic pattern reveals length-dependent service.

As shown in Figure~\ref{fig:length}, the three cities exhibit a decisive
reversal. In Manhattan, SP follows a U-shaped profile and rises sharply over
the longest deciles. PIBT and GLC vary less, but both still impose higher
slowdown on later deciles than on the shortest trips. Chicago produces a
different ordering: SP trends upward overall, whereas PIBT and GLC impose
their greatest slowdown on the shortest trips and exhibit a non-monotonic
spike around the seventh decile. San Francisco shows the sharpest split. SP
increasingly penalizes longer trips, while PIBT and GLC begin with high
short-trip slowdown and decline overall with trip length. These within-city
disagreements establish that the direction and shape of trip-length inequity
depend jointly on the city and coordinator; no trip-length group is
universally protected or disadvantaged.

Two mechanisms produce these reversals. Longer routes encounter more
opportunities for congestion, whereas the same absolute wait creates a larger
slowdown for shorter trips. Network topology, routing, and contention
determine which mechanism dominates. Because every city--coordinator pair
induces its own length profile, the spatial analysis below removes this effect
before comparing origin and destination regions.

\subsection{Spatial Inequity Grows with Demand}
\label{sec:spatial}
\label{sec:phenomenon}
\label{sec:mechanism}
\label{sec:taxonomy}

We remove each experimental condition's length profile using the residual
slowdown in Equation~\ref{eq:si}, then group trips by origin or destination
taxi zone. This isolates systematic geographic disparity among trips with
comparable free-flow requirements.

As shown in Figure~\ref{fig:map} and Appendix
Figures~\ref{fig:app-spatial-maps-sp} and~\ref{fig:app-spatial-maps-pibt},
the origin-group residuals under GLC, SP, and PIBT form spatially coherent
clusters in every city: some regions repeatedly receive worse service than
their trip lengths predict, while others receive better service. Positive and
negative residuals concentrate in neighboring zones rather than isolated
outliers, showing that the disparity is geographically organized and recurs
across coordinators and road networks.

\begin{figure}[t]
\centering
\includegraphics[width=\linewidth]{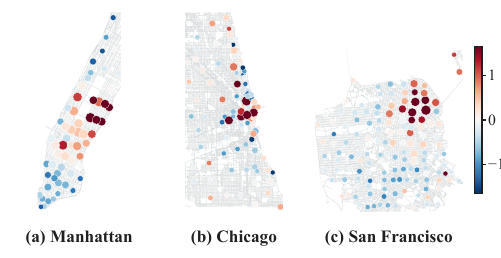}
				\vspace{-2em}
\Description{Three enlarged maps show GLC origin-region residual slowdown in
Manhattan, Chicago, and San Francisco. Red and blue indicate worse and better
service than trips of comparable free-flow length.}
\caption{Origin-side spatial inequity under GLC across cities.}
\label{fig:map}
\end{figure}

As shown in Figure~\ref{fig:density-onset}, spatial inequity increases
with demand across SP, PIBT, and GLC in all three cities. Origin grouping
consistently produces larger disparity than destination grouping, and the gap
widens as congestion intensifies. At lighter demand, origin- and
destination-side disparities are relatively close; as more vehicles compete
for the same road capacity, both rise and the origin-side curves separate
more strongly. Congestion exposes disparities hidden by
fleet-wide averages, with origins defining their spatial pattern.

\begin{figure}[t]
\centering
\includegraphics[width=\linewidth]{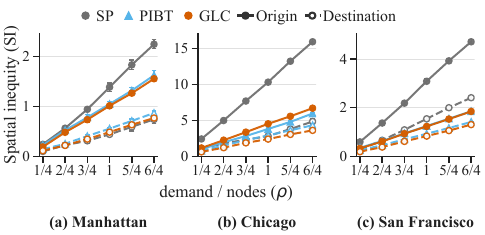}
\vspace{-2em}
\Description{Three aligned panels compare origin- and destination-region
spatial inequity for SP, PIBT, and GLC in Manhattan, Chicago, and San
Francisco. Solid lines with filled markers denote origin grouping; dashed
lines with hollow markers denote destination grouping; error bars show 95
percent confidence intervals over ten fixed demand subsamples.}
\caption{SI grows with demand across cities, led by origin.}
\label{fig:density-onset}

\end{figure}

Together, Figures~\ref{fig:length} and~\ref{fig:density-onset} show that
trip-length inequity depends on the city and coordinator, whereas congestion
consistently amplifies spatial inequity, especially across origins. This
evidence motivates a targeted intervention that redirects delayed vehicles
using realized congestion while explicitly bounding online planning.
 
 \section{SPARE: Budgeted Congestion-Responsive Fleet Coordination}
\label{sec:mitigation}

SPARE treats online coordination as the repeated allocation of a limited
planning budget. At each review, it decides which delayed vehicles should be
reconsidered and where their guide paths should go under the waiting pressure
observed during execution. The method couples two decisions:
delay-prioritized vehicle selection and congestion-responsive route
adaptation. The movement resolver of Section~\ref{sec:model} remains
responsible for producing capacity-feasible moves; SPARE determines the guide
paths that inform those moves. Our experiments instantiate the resolver with
GLC, while the formulation applies to any coordinator with a guide-path
interface.

\subsection{Budgeted Online Coordination}
\label{sec:budgeted-coordination}
\label{sec:module}

SPARE maintains a cumulative delay counter $\tau_i$ for every vehicle and a
node-level waiting field $c[v]$. Whenever vehicle $i$ is unable to move at a
step, $\tau_i$ increases by one and the waiting observed at its current node
is added to $c$. The counter $\tau_i$ represents the service loss accumulated
by the vehicle, while $c$ records where the fleet has recently experienced
contention. Selection does not use a trip's free-flow length or length-bin
membership: intervention priority is determined by delay observed during
execution.

Reviews occur every $K$ steps. Index them by $q=1,2,\ldots$, let $A_q$ be the
active vehicles at review $q$, and let
$D_q=\{i\in A_q:\tau_i>0\}$ be those that have experienced delay. SPARE may
replan at most $R$ vehicles, so it sets
$m_q=\min\{R,|D_q|\}$ and selects
\begin{equation}
B_q\in
\arg\max_{\substack{B\subseteq D_q\\|B|=m_q}}
\sum_{i\in B}\tau_i.
\label{eq:selection}
\end{equation}
Thus the available planning budget is assigned to the vehicles carrying the
largest accumulated delay.

Before constructing new paths, SPARE geometrically decays the waiting field
as $c[v]\leftarrow\tfrac{1}{2}c[v]$ and freezes the resulting values as
$c_q$. The corresponding edge-weight snapshot is
$w_q(u,v)=1+c_q[v]$.
For the current position $p_i$ and goal $g_i$, let
$\mathcal{P}_G(p_i,g_i)$ be the set of directed paths between them. The cost
of a path $P$ under the review snapshot is
\begin{equation}
\begin{aligned}
C_{w_q}(P)
  &=\sum_{(u,v)\in P}w_q(u,v)\\
  &=|P|+\sum_{(u,v)\in P}c_q[v].
\end{aligned}
\label{eq:path-cost}
\end{equation}
The first term measures free-flow path length; the second measures exposure
to recently observed waiting. For every selected vehicle, SPARE computes
\begin{equation}
P_{i,q}^{+}\in
\arg\min_{P\in\mathcal{P}_G(p_i,g_i)}C_{w_q}(P).
\label{eq:local-optimality}
\end{equation}
The review interval $K$ controls how often the system reacts, while $R$
controls how much of the fleet can be reconsidered at one review.

Algorithm~\ref{alg:mitigation} gives the complete procedure. All vehicles
begin with free-flow shortest-path guidance and zero-valued online signals
(lines 1--2). At each review, lines 6--12 decay the waiting field, identify
the delayed vehicles covered by the budget, and replace only their guide
paths according to Equations~\ref{eq:selection}
and~\ref{eq:local-optimality}; vehicles outside $B$ retain their current
guidance. Line 14 invokes the movement resolver, which returns both
capacity-feasible next positions and the advanced or repaired remaining
paths. Lines 15--20 then record newly observed waits and update the execution
state. The waiting field is decayed because it represents recent network
pressure, whereas $\tau_i$ remains cumulative because it represents the
service loss carried by vehicle $i$ across reviews.

\begin{algorithm}[t]
\caption{SPARE: budgeted online fleet coordination.}
\label{alg:mitigation}
\begin{algorithmic}[1]
\Require road graph $G=(V,E)$; origins $s_i$ and goals $g_i$ for $i=1{\ldots}N$; review interval $K$; reroute budget $R$
\State $p_i\gets s_i$; \quad $P_i\gets\textsc{ShortestPath}(G,s_i,g_i)$ for all $i$
\State $\tau_i\gets0$ for all $i$; \quad $c[v]\gets0$ for all $v$
\For{$t=1,2,\ldots$ until all vehicles reach their goals}
  \State $A\gets\{i:p_i\ne g_i\}$
  \If{$t\bmod K=0$}
    \State $c[v]\gets\tfrac12c[v]$ for all $v$
    \State $w(u,v)\gets1+c[v]$ for all $(u,v)\in E$
    \State $D\gets\{i\in A:\tau_i>0\}$; \quad $m\gets\min\{R,|D|\}$
    \State $B\gets$ the $m$ vehicles in $D$ with largest $\tau_i$
    \ForAll{$i\in B$}
      \State $P_i\gets\textsc{ShortestPath}(G,p_i,g_i;w)$
    \EndFor
  \EndIf
  \State $(\{p'_i\},\{\widetilde P_i\})\gets
  \textsc{ResolveStep}(G,\{p_i\}_{i\in A},\{g_i\},\{P_i\})$
  \ForAll{$i\in A$}
    \If{$p'_i=p_i$}
      \State $\tau_i\gets\tau_i+1$; \quad $c[p_i]\gets c[p_i]+1$
    \EndIf
    \State $p_i\gets p'_i$; \quad $P_i\gets\widetilde P_i$
  \EndFor
\EndFor
\end{algorithmic}
\end{algorithm}

\subsection{Guarantees and Complexity}
\label{sec:spare-guarantees}

\begin{theorem}[Per-review budget guarantee]
\label{thm:spare-budget}
Assume that every active vehicle's goal is reachable from its current
position. At review $q$, SPARE's selected set $B_q$ maximizes the total
accumulated delay addressed by any subset of $m_q$ delayed vehicles, and the
new guide path of every $i\in B_q$ minimizes $C_{w_q}$ over
$\mathcal{P}_G(p_i,g_i)$. Moreover, after the initial guide paths are
constructed, any execution of $T$ steps performs at most
$R\lfloor T/K\rfloor$ additional shortest-path computations.
\end{theorem}

\begin{proof}
Selecting the $m_q$ largest values of $\tau_i$ maximizes their sum over all
size-$m_q$ subsets of $D_q$, proving the selection claim. With $c_q$ fixed,
all weights $w_q(u,v)=1+c_q[v]$ are positive, and the shortest-path
computation in Algorithm~\ref{alg:mitigation} therefore returns a minimizer
of $C_{w_q}$ for each selected vehicle. Finally, reviews occur only at steps
divisible by $K$, giving at most $\lfloor T/K\rfloor$ reviews, and each
review computes at most $R$ new paths.
\end{proof}

Let $P_{i,q}^{-}$ be vehicle $i$'s valid remaining guide path immediately
before review $q$. Equation~\ref{eq:local-optimality} directly gives
\begin{equation}
C_{w_q}(P_{i,q}^{+})\le C_{w_q}(P_{i,q}^{-}).
\label{eq:local-improvement}
\end{equation}
This is a per-review decision guarantee, not a claim that one reroute must
reduce realized travel time or spatial inequity: later interactions can
change the congestion encountered by every vehicle. The global efficiency
and distributional effects are therefore evaluated empirically.

\paragraph{Time complexity.}
With adjacency lists and binary-heap Dijkstra, constructing the initial
guide paths costs $O(N(|V|+|E|)\log|V|)$. At one review, decaying the waiting
field and materializing the edge weights costs $O(|V|+|E|)$, selecting the
largest counters with a size-$R$ min-heap costs $O(N\log(R+1))$, and the
route updates cost $O(R(|V|+|E|)\log|V|)$. Excluding the movement resolver,
whose cost is shared by all compared route-adaptation policies, SPARE's
amortized online work per step is
\begin{equation}
O\!\left(
N+
\frac{
R(|V|+|E|)\log|V|
+|V|+|E|+N\log(R+1)
}{K}
\right).
\label{eq:amortized-complexity}
\end{equation}
The $O(N)$ term updates the two online signals once per active vehicle and
can be fused with the resolver's existing vehicle loop. In contrast,
full-fleet replanning invokes up to $N$ route computations at every step.

\paragraph{Space complexity.}
The waiting field uses $O(|V|)$ space, the counters and selection heap use
$O(N+R)$, and edge weights use $O(|E|)$. Recomputed guide paths replace paths
maintained by the coordination state, so SPARE adds
$O(|V|+|E|+N)$ space.
 \newcommand{\best}[1]{{\bfseries\boldmath $#1$}}

\begin{table*}[ht]
	\centering
	\small
	\setlength{\tabcolsep}{1.3pt}
	\renewcommand{\arraystretch}{1}
	\caption{Cross-city comparison at the node-normalized demand setting $\rho=1$. Lower is better, including CPU time; bold marks the best result. OOT denotes out of time, exceeding 2 hours.}
	\label{tab:main}
	\begin{tabular}{@{}l|lrrrrrr@{}}
		\toprule
		\textbf{City} & \textbf{Method} & \textbf{Overhead (\%)} & \textbf{Mean slowdown} & \textbf{LGD} & \textbf{Origin SI} & \textbf{Destination SI} & \textbf{CPU time (s)} \\
		\midrule
		\multirow{7}{*}{\textbf{Manhattan}} & SP~\cite{poa,congestiongames} & $367.9 \pm 8.3$ & $4.08 \pm 0.07$ & $1.403 \pm 0.065$ & $1.386 \pm 0.126$ & $0.442 \pm 0.074$ & $4.0 \pm 0.9$ \\
		& GSP~\cite{poa,congestiongames} & $324.9 \pm 6.3$ & $3.70 \pm 0.06$ & $1.298 \pm 0.073$ & $1.112 \pm 0.039$ & $0.372 \pm 0.042$ & \best{3.7 \pm 0.5} \\
		& TAP~\cite{wardrop,sheffi} & $289.8 \pm 23.4$ & $3.61 \pm 0.21$ & $0.746 \pm 0.090$ & $0.973 \pm 0.112$ & $0.471 \pm 0.154$ & $974.3 \pm 257.1$ \\
		& PIBT~\cite{pibt} & $251.6 \pm 7.4$ & $3.69 \pm 0.06$ & $0.699 \pm 0.050$ & $1.047 \pm 0.078$ & $0.558 \pm 0.031$ & $57.1 \pm 22.8$ \\
		& FAIR-RD~\cite{fairmapf} & $292.6 \pm 6.2$ & $4.56 \pm 0.07$ & $1.588 \pm 0.057$ & $1.316 \pm 0.096$ & $0.761 \pm 0.078$ & $41.6 \pm 18.8$ \\
		& GLC~\cite{li2026local} & $256.5 \pm 7.4$ & $3.45 \pm 0.07$ & $0.381 \pm 0.037$ & $1.013 \pm 0.078$ & $0.487 \pm 0.045$ & $40.4 \pm 15.6$ \\
		& \cellcolor{gray!20}\textbf{SPARE} & \cellcolor{gray!20}\best{180.7 \pm 2.7} & \cellcolor{gray!20}\best{2.84 \pm 0.02} & \cellcolor{gray!20}\best{0.123 \pm 0.017} & \cellcolor{gray!20}\best{0.561 \pm 0.042} & \cellcolor{gray!20}\best{0.262 \pm 0.015} & \cellcolor{gray!20}$86.9 \pm 3.7$ \\
		\midrule
		\multirow{7}{*}{\textbf{Chicago}} & SP~\cite{poa,congestiongames} & $3222.5 \pm 106.1$ & $34.36 \pm 1.06$ & $7.080 \pm 0.355$ & $10.323 \pm 0.280$ & $2.828 \pm 0.383$ & $115.6 \pm 46.5$ \\
		& GSP~\cite{poa,congestiongames} & $3153.4 \pm 58.7$ & $31.84 \pm 0.57$ & $7.888 \pm 0.384$ & $9.996 \pm 0.220$ & $3.550 \pm 0.314$ & \best{99.6 \pm 43.7} \\
		& TAP~\cite{wardrop,sheffi} & \textsc{OOT} & \textsc{OOT} & \textsc{OOT} & \textsc{OOT} & \textsc{OOT} & \textsc{OOT} \\
		& PIBT~\cite{pibt} & $1540.8 \pm 14.6$ & $19.17 \pm 0.21$ & $6.735 \pm 0.135$ & \best{3.825 \pm 0.185} & $2.931 \pm 0.121$ & $1811.9 \pm 834.5$ \\
		& FAIR-RD~\cite{fairmapf} & $1622.2 \pm 8.3$ & $20.51 \pm 0.15$ & $7.858 \pm 0.113$ & $5.054 \pm 0.238$ & $3.833 \pm 0.128$ & $814.5 \pm 303.0$ \\
		& GLC~\cite{li2026local} & $1463.4 \pm 10.9$ & $17.31 \pm 0.13$ & $5.672 \pm 0.130$ & $4.540 \pm 0.224$ & $2.420 \pm 0.104$ & $1073.4 \pm 412.2$ \\
		& \cellcolor{gray!20}\textbf{SPARE} & \cellcolor{gray!20}\best{1436.3 \pm 9.8} & \cellcolor{gray!20}\best{17.09 \pm 0.15} & \cellcolor{gray!20}\best{5.672 \pm 0.164} & \cellcolor{gray!20}$4.163 \pm 0.167$ & \cellcolor{gray!20}\best{2.351 \pm 0.090} & \cellcolor{gray!20}$1021.1 \pm 145.3$ \\
		\midrule
		\multirow{7}{*}{\textbf{San Francisco}} & SP~\cite{poa,congestiongames} & $727.2 \pm 24.9$ & $8.23 \pm 0.27$ & $2.031 \pm 0.149$ & $3.087 \pm 0.155$ & $1.547 \pm 0.068$ & $20.9 \pm 6.9$ \\
		& GSP~\cite{poa,congestiongames} & $662.2 \pm 14.5$ & $7.69 \pm 0.17$ & $2.044 \pm 0.143$ & $2.905 \pm 0.077$ & $1.345 \pm 0.057$ & \best{13.2 \pm 5.0} \\
		& TAP~\cite{wardrop,sheffi} & \textsc{OOT} & \textsc{OOT} & \textsc{OOT} & \textsc{OOT} & \textsc{OOT} & \textsc{OOT} \\
		& PIBT~\cite{pibt} & $318.6 \pm 4.1$ & $5.02 \pm 0.06$ & $1.801 \pm 0.070$ & $1.250 \pm 0.030$ & $0.933 \pm 0.023$ & $264.7 \pm 75.8$ \\
		& FAIR-RD~\cite{fairmapf} & $341.2 \pm 3.8$ & $5.68 \pm 0.06$ & $2.656 \pm 0.079$ & $1.275 \pm 0.045$ & $1.059 \pm 0.029$ & $101.9 \pm 35.7$ \\
		& GLC~\cite{li2026local} & $310.3 \pm 3.7$ & $4.60 \pm 0.05$ & $1.343 \pm 0.040$ & $1.233 \pm 0.029$ & $0.837 \pm 0.023$ & $139.6 \pm 46.0$ \\
		& \cellcolor{gray!20}\textbf{SPARE} & \cellcolor{gray!20}\best{303.2 \pm 3.3} & \cellcolor{gray!20}\best{4.51 \pm 0.05} & \cellcolor{gray!20}\best{1.298 \pm 0.037} & \cellcolor{gray!20}\best{1.199 \pm 0.027} & \cellcolor{gray!20}\best{0.806 \pm 0.024} & \cellcolor{gray!20}$226.0 \pm 10.3$ \\
		\bottomrule
	\end{tabular}
\end{table*}

\section{Experimental Evaluation}
\label{sec:evaluation}

We evaluate SPARE on the Manhattan, Chicago, and San Francisco networks using
their matched taxi-demand records. The cross-city comparison uses $\rho=1$,
while the Manhattan scalability sweep varies $\rho$ from $1/4$ to $6/4$.
The suite contains six baselines---SP and
GSP~\cite{poa,congestiongames}, TAP~\cite{wardrop,sheffi},
PIBT~\cite{pibt}, FAIR-RD~\cite{fairmapf}, and GLC~\cite{li2026local}---and
the fairness comparison additionally sweeps
ITAP-$\alpha$~\cite{itap}. Within each condition, all methods receive
identical OD draws and release times.

\paragraph{Implementation details.}
SPARE uses one fixed, city-specific operating point selected from the sweeps
in Section~\ref{sec:budget}; baseline settings follow their original papers.
Tables report mean $\pm$ population standard deviation over ten fixed demand
subsamples. All runs are single-threaded Python simulations executed on four
machines, each allocating 32 vCPUs of an AMD EPYC 9654 processor and 60\,GB
of RAM. Computational analyses report per-run CPU time.

\subsection{Main Results}
\label{sec:density}

Table~\ref{tab:main} compares all methods across the three cities at
$\rho=1$. Table~\ref{tab:frontier} isolates the comparison with methods that
pursue fairness explicitly, and Figure~\ref{fig:density} tracks effectiveness,
fairness, and computation as Manhattan demand grows.

\subsubsection{Effectiveness and Fairness.}
Overhead and mean slowdown capture fleet efficiency, while LGD and the two SI
measures test whether the remaining delay is concentrated by trip length or
geography. Reading them together prevents an apparent efficiency gain from
hiding a worse distribution of service. As shown in Table~\ref{tab:main}, at
$\rho=1$ SPARE achieves the best result on all five outcome metrics in
Manhattan and San Francisco. In Chicago, it achieves the lowest overhead,
mean slowdown, and Destination SI, matches GLC's LGD at the reported
precision, and reduces Origin SI relative to GLC. PIBT attains a lower Origin
SI in Chicago, but with higher fleet overhead, mean slowdown, and CPU time.

The variation follows how much congestion can be avoided by changing guide
paths. Unlike fixed, static, or priority-only interventions, SPARE redirects
delayed vehicles using realized waiting pressure. This produces the largest
separation in Manhattan, where concentrated bottlenecks leave exploitable
alternatives; the gains are narrower on the larger Chicago and San Francisco
networks. Across all three settings, SPARE delivers the strongest joint
efficiency--fairness profile rather than shifting delay to another
trip-length group or region.

\begin{table}[H]
	\centering
	\normalsize
	\setlength{\tabcolsep}{1.1pt}
	\renewcommand{\arraystretch}{1.03}
	\newcommand{\std}[1]{\mathbin{\pm}{\scriptstyle #1}}
	\caption{Explicit fairness comparison on Manhattan demand.}
	\label{tab:frontier}
	\resizebox{\columnwidth}{!}{\begin{tabular}{lrrrrr}
			\toprule
			\textbf{Method} & \textbf{OH (\%)} & \textbf{Mean} & \textbf{LGD} & \textbf{O-SI} & \textbf{D-SI} \\
			\midrule
			\multicolumn{6}{@{}l}{\textit{Assignment-based fairness methods}} \\
			ITAP-$0$ & $289.8 \std{23.4}$ & $3.61 \std{0.21}$ & $0.746 \std{0.090}$ & $0.973 \std{0.112}$ & $0.471 \std{0.154}$ \\
			ITAP-$0.25$ & $276.8 \std{20.0}$ & $3.49 \std{0.17}$ & $0.727 \std{0.105}$ & $0.937 \std{0.116}$ & $0.430 \std{0.166}$ \\
			ITAP-$0.5$ & $274.9 \std{13.7}$ & $3.50 \std{0.14}$ & $0.658 \std{0.060}$ & $0.981 \std{0.093}$ & $0.402 \std{0.079}$ \\
			ITAP-$0.75$ & $277.8 \std{15.7}$ & $3.52 \std{0.13}$ & $0.684 \std{0.071}$ & $0.979 \std{0.065}$ & $0.362 \std{0.089}$ \\
			ITAP-$1$ & $284.2 \std{24.1}$ & $3.57 \std{0.25}$ & $0.707 \std{0.103}$ & $1.003 \std{0.111}$ & $0.374 \std{0.138}$ \\
			\midrule
			\multicolumn{6}{@{}l}{\textit{Routing and coordination baselines}} \\
			FAIR-RD & $292.6 \std{6.2}$ & $4.56 \std{0.07}$ & $1.588 \std{0.057}$ & $1.316 \std{0.096}$ & $0.761 \std{0.078}$ \\
			GLC & $256.5 \std{7.4}$ & $3.45 \std{0.07}$ & $0.381 \std{0.037}$ & $1.013 \std{0.078}$ & $0.487 \std{0.045}$ \\
			\midrule
			\rowcolor{gray!20}
			\textbf{SPARE (ours)} & $\mathbf{180.7} \std{2.7}$ & $\mathbf{2.84} \std{0.02}$ & $\mathbf{0.123} \std{0.017}$ & $\mathbf{0.561} \std{0.042}$ & $\mathbf{0.262} \std{0.015}$ \\
			\bottomrule
		\end{tabular}
	}
\end{table}
\begin{figure*}[t]
	\centering
	\vspace{-1em}
	\includegraphics[width=\linewidth]{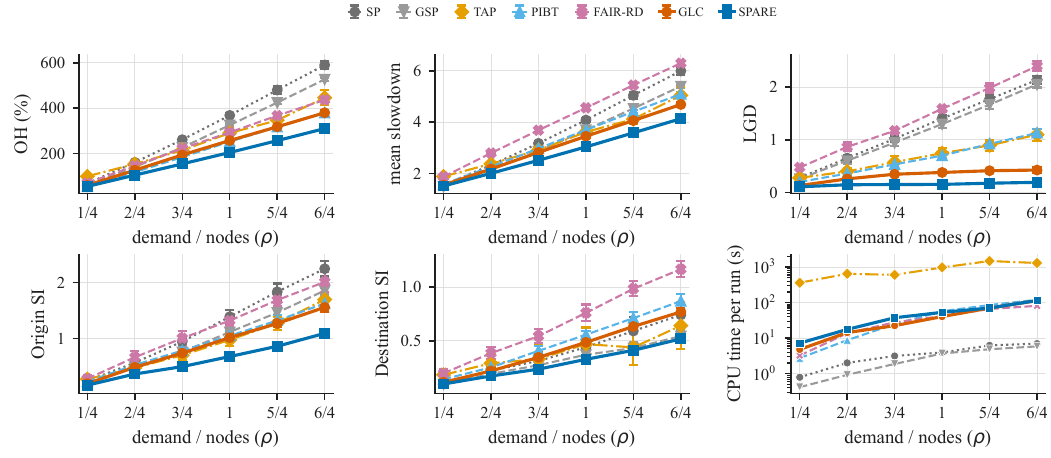}
	\vspace{-2em}
	\caption{Scalability with load ratio on Manhattan demand. }
	\label{fig:density}
\end{figure*}
\subsubsection{Comparison with Explicit Fairness Methods.}
\label{sec:faircomp}
We compare SPARE with two representative fairness interventions. FAIR-RD
prioritizes accumulated delay without changing routes, while
ITAP-$\alpha$~\cite{itap} interpolates static assignment between user
equilibrium and the marginal-cost system optimum. We sweep five ITAP settings
on matched OD samples. Table~\ref{tab:frontier} reports Manhattan. As shown in
Appendix Tables~\ref{tab:app-frontier-chicago}
and~\ref{tab:app-frontier-sf}, SPARE also has the lowest overhead and mean
slowdown among the completed methods in Chicago and San Francisco.

As shown in Table~\ref{tab:frontier}, the ITAP results remain within a narrow
range, while SPARE achieves the lowest overhead, mean slowdown, LGD, Origin
SI, and Destination SI across every ITAP setting, FAIR-RD, and GLC. Static
assignment cannot remove disparities created during execution, and FAIR-RD
changes who moves first rather than the congested route producing the delay.
The comparison therefore establishes the advantage of altering realized route
exposure.

\subsubsection{Scalability with Fleet Size.}
\label{sec:scalability}
We evaluate six Manhattan load ratios from $\rho=1/4$ to $\rho=6/4$ under
matched demand subsamples. As shown in Figure~\ref{fig:density} and Appendix
Figures~\ref{fig:app-density-scalability-chicago}
and~\ref{fig:app-density-scalability-sf}, overhead, slowdown, LGD, and SI rise
with load across all three cities, while SPARE maintains favorable efficiency
and disparity scaling. Its reviews add computation relative to SP and GSP,
yet remain tractable because each review reroutes only a bounded subset. TAP
completes the Manhattan sweep, but
Table~\ref{tab:main} shows that it does not complete at $\rho=1$ on the larger
Chicago and San Francisco networks, where repeated global assignment becomes
prohibitive. SPARE completes the full cross-city and load-ratio evaluation
because its planning work is explicitly controlled by the review interval and
reroute budget.

\begin{table*}[!t]
	\centering
	\small
	\caption{Controlled ablation at $\rho=1$. Entries are mean $\pm$ standard
	deviation over ten fixed demand subsamples; lower is better, and bold marks
	the best result within each city and metric.}
	\label{tab:ablation-four-variants}
	\begin{tabular}{llrrrrr}
		\toprule
		\textbf{City} & \textbf{Variant} & \textbf{Overhead (\%)} & \textbf{Mean slowdown} & \textbf{LGD} & \textbf{Origin SI} & \textbf{Destination SI} \\
		\midrule
		\textbf{Manhattan} & \textbf{SPARE} & $\mathbf{180.7 \pm 2.7}$ & $\mathbf{2.84 \pm 0.02}$ & $\mathbf{0.123 \pm 0.017}$ & $0.561 \pm 0.042$ & $\mathbf{0.262 \pm 0.015}$ \\
		& w/o periodic rerouting & $256.5 \pm 7.8$ & $3.45 \pm 0.07$ & $0.381 \pm 0.039$ & $1.013 \pm 0.082$ & $0.487 \pm 0.047$ \\
		& w/o congestion weighting & $254.6 \pm 7.6$ & $3.43 \pm 0.07$ & $0.388 \pm 0.042$ & $0.984 \pm 0.062$ & $0.471 \pm 0.040$ \\
		& w/o delay ranking & $182.3 \pm 2.1$ & $2.85 \pm 0.02$ & $0.131 \pm 0.014$ & $\mathbf{0.553 \pm 0.004}$ & $0.269 \pm 0.013$ \\
		\midrule
		\textbf{Chicago}  & \textbf{SPARE} & $\mathbf{1436.3 \pm 9.8}$ & $\mathbf{17.09 \pm 0.15}$ & $5.672 \pm 0.164$ & $\mathbf{4.163 \pm 0.167}$ & $2.351 \pm 0.090$ \\
		& w/o periodic rerouting & $1462.6 \pm 11.9$ & $17.29 \pm 0.14$ & $5.645 \pm 0.139$ & $4.522 \pm 0.238$ & $2.400 \pm 0.109$ \\
		& w/o congestion weighting & $1460.8 \pm 12.5$ & $17.26 \pm 0.17$ & $\mathbf{5.637 \pm 0.160}$ & $4.519 \pm 0.299$ & $2.399 \pm 0.118$ \\
		& w/o delay ranking & $1581.6 \pm 8.8$ & $19.18 \pm 0.07$ & $5.844 \pm 0.098$ & $4.217 \pm 0.243$ & $\mathbf{2.332 \pm 0.100}$ \\
		\midrule
		\textbf{San Francisco}  & \textbf{SPARE} & $\mathbf{303.2 \pm 3.3}$ & $\mathbf{4.51 \pm 0.05}$ & $\mathbf{1.298 \pm 0.037}$ & $\mathbf{1.199 \pm 0.027}$ & $\mathbf{0.806 \pm 0.024}$ \\
		& w/o periodic rerouting & $310.3 \pm 3.9$ & $4.60 \pm 0.06$ & $1.343 \pm 0.042$ & $1.233 \pm 0.030$ & $0.837 \pm 0.024$ \\
		& w/o congestion weighting & $309.7 \pm 3.6$ & $4.59 \pm 0.05$ & $1.340 \pm 0.043$ & $1.232 \pm 0.028$ & $0.832 \pm 0.025$ \\
		& w/o delay ranking & $306.2 \pm 1.7$ & $4.56 \pm 0.02$ & $1.313 \pm 0.035$ & $1.218 \pm 0.016$ & $0.819 \pm 0.008$ \\
		\bottomrule
	\end{tabular}
\end{table*}
\subsection{Ablation Study}
\label{sec:ablation}

We use a nested ablation at $\rho=1$ in all three cities under matched OD
draws. \textbf{w/o periodic rerouting} retains the initial paths,
\textbf{w/o congestion weighting} uses free-flow costs, and \textbf{w/o delay
ranking} samples uniformly from the same positive-delay set and budget; these
variants isolate online adaptation, routing, and vehicle selection,
respectively. As shown in Table~\ref{tab:ablation-four-variants}, periodic
rerouting and congestion weighting drive the Manhattan gains, improve
efficiency and both SI metrics in Chicago, and improve all five metrics in San
Francisco. Delay ranking improves every metric except Manhattan Origin SI and
Chicago Destination SI, including all five metrics in San Francisco.
These results show that rerouting alone is insufficient: its gains depend on
where reroutes are sent and which vehicles receive them.
Together, the three components deliver the strongest joint
efficiency--fairness profile by adapting online, avoiding congested corridors,
and directing the limited budget to the most delayed vehicles.

\subsection{Parameter Sensitivity}
\label{sec:budget}

At $\rho=1$, we sweep $K$ with $R=400$ and $R$ with $K=15$.
Figure~\ref{fig:paramsens} reports Manhattan, while Appendix
Figures~\ref{fig:app-paramsens-chicago} and~\ref{fig:app-paramsens-sf} report
Chicago and San Francisco.
All three figures show overhead and per-run CPU time.

\subsubsection{Review Interval $K$.}
With $R=400$, we sweep $K$ from frequent to infrequent review. As shown in
Figure~\ref{fig:paramsens} and Appendix
Figures~\ref{fig:app-paramsens-chicago} and~\ref{fig:app-paramsens-sf}, Manhattan has an interior
low-overhead region, while Chicago's overhead changes modestly as CPU time
falls with less frequent review. San Francisco likewise reaches its lowest
overhead at an intermediate interval rather than at the most frequent review.
Thus, the effect of $K$ is city-dependent and non-monotonic: reviewing too
frequently can react before a stable congestion signal forms, whereas sparse
reviews can leave established bottlenecks uncorrected.

\subsubsection{Reroute Budget $R$.}
With $K=15$, we sweep $R$ from no rerouting to large per-review budgets. As
shown in Figure~\ref{fig:paramsens} and Appendix
Figures~\ref{fig:app-paramsens-chicago} and~\ref{fig:app-paramsens-sf}, Manhattan exhibits diminishing
overhead returns, and Chicago obtains a smaller but continuing reduction as
$R$ grows. San Francisco's overhead also continues to decrease through the
largest tested budget, $R=3000$; its curve does not support an intermediate
saturation claim. The highest-delay vehicles are selected first, so the
marginal value of additional reroutes can vary by city and budget.
Together, the sweeps establish $K$ and $R$ as city-dependent controls over the
delay--CPU trade-off.

\begin{figure}[t]
  \centering
  \includegraphics[width=\linewidth]{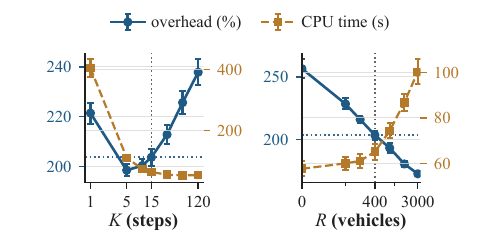}
  \Description{Two side-by-side Manhattan panels show SPARE's sensitivity to review interval K with R fixed at 400 and reroute budget R with K fixed at 15. Blue curves show overhead and brown dashed curves show process CPU time. Dotted crosshairs identify the fixed reference setting, while the ordinary curve markers show its observations.}
  \vspace{-2.5em}
  \caption{Parameter sensitivity on Manhattan (mean $\pm$ SD over ten
  subsamples).}
  \label{fig:paramsens}
\end{figure}

  \subsection{Case Studies}
  \label{sec:case}

  To test whether SPARE's aggregate gains are shared rather than shifted onto a small subset of trips, we compare trip-level slowdown inequality against representative baselines: GSP, FAIR-RD, and GLC.
  
  Three metrics measure how evenly slowdown is shared among trips. The \textbf{Gini coefficient} is low when slowdowns are similar and rises when a minority bears a disproportionate burden; the \textbf{coefficient of variation (CoV)} is trip-to-trip spread relative to the fleet mean; and the \textbf{Theil index} is more sensitive to a small set of exceptionally delayed trips. Two metrics capture its geographic structure: \textbf{B-Theil} is the total-Theil component due to differences between origin-region means rather than within regions, and \textbf{W/B} divides the worst eligible origin-region mean by the best (e.g., W/B$=2$ means twice the average slowdown).
  Table~\ref{tab:individual} shows that SPARE is lowest on all five metrics. Its Gini and CoV are about $25\%$ below the strongest competing results; the larger Theil ($43\%$) and B-Theil ($55\%$) reductions indicate fewer extreme trip-level burdens and systematic origin-region gaps. W/B falls from $2.44$ to $1.99$, bringing the worst-served origin region closer to the best.

  \begin{table}[!t]
  \centering
  \small
  \caption{Trip-level slowdown inequality on Manhattan at $\rho=1$. B-Theil denotes the between-origin-region Theil component, and W/B denotes the worst-to-best origin-region mean-slowdown ratio.}
  \label{tab:individual}
  \setlength{\tabcolsep}{2pt}
  \renewcommand{\arraystretch}{1.04}
  \newcommand{\casestd}[1]{\mathbin{\pm}{\scriptstyle #1}}
  \begin{tabular}{lrrrrr}
  \toprule
  \textbf{Method} & \textbf{Gini} & \textbf{CoV} & \textbf{Theil} & \textbf{B-Theil} & \textbf{W/B} \\
  \midrule
  GSP & $0.368 \casestd{0.005}$ & $0.707 \casestd{0.015}$ & $0.217 \casestd{0.007}$ & $0.020 \casestd{0.003}$ & $2.52 \casestd{0.12}$ \\
  FAIR-RD & $0.383 \casestd{0.005}$ & $0.824 \casestd{0.022}$ & $0.254 \casestd{0.009}$ & $0.026 \casestd{0.003}$ & $3.42 \casestd{0.53}$ \\
  GLC & $0.365 \casestd{0.007}$ & $0.762 \casestd{0.024}$ & $0.227 \casestd{0.011}$ & $0.021 \casestd{0.002}$ & $2.44 \casestd{0.22}$ \\
  \rowcolor{gray!20}
  \textbf{SPARE} & $\mathbf{0.275} \casestd{0.003}$ & $\mathbf{0.533} \casestd{0.005}$ & $\mathbf{0.124} \casestd{0.002}$ & $\mathbf{0.009} \casestd{0.001}$ & $\mathbf{1.99} \casestd{0.19}$ \\
  \bottomrule
  \end{tabular}
  \end{table}

\section{Conclusion}
\label{sec:conclusion}
Across three distinct datasets combining real-city road networks and taxi demand, we uncover context-dependent trip-length inequity and demand-amplified spatial inequity, particularly across origin regions. Against six representative baselines, SPARE consistently delivers the strongest joint efficiency--fairness performance through scalable, budgeted, congestion-responsive rerouting without full-fleet replanning.

\section*{Limitations and Ethical Considerations}
\label{sec:limitations}

We use public taxi records under their release terms and report only aggregate
results without personal identifiers. Historical demand may encode service
bias, and our metrics capture operational rather than demographic fairness.
Simultaneous release and lane-based capacities abstract time-varying traffic.
Deployment requires local calibration, privacy and safety review, and
continued fairness monitoring.

\section*{Generative AI Usage}

GPT and Claude were used for language polishing.

\clearpage
\appendix

\setcounter{dbltopnumber}{2}
\renewcommand{\dbltopfraction}{0.95}
\renewcommand{\dblfloatpagefraction}{0.70}

\section{Additional Spatial-Diagnostic Figures}
\label{app:spatial-figures}

Figure~\ref{fig:map} shows GLC across all three cities, and
Figure~\ref{fig:density-onset} gives the three-city endpoint-by-method trend.
Figures~\ref{fig:app-spatial-maps-sp} and~\ref{fig:app-spatial-maps-pibt}
extend the origin-side maps to two additional baselines, with the same
Manhattan, Chicago, and San Francisco columns in both figures.

\begin{figure}[H]
\centering
\includegraphics[width=\linewidth]{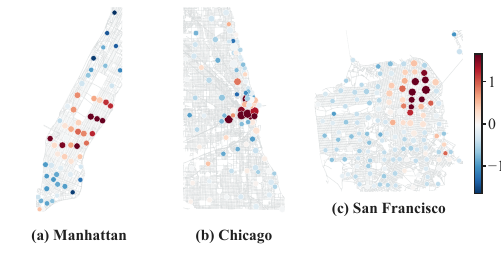}
\Description{Three aligned origin-disparity maps show SP in Manhattan,
Chicago, and San Francisco. Colors are within-panel standardized
length-adjusted slowdown.}
\caption{Origin-disparity maps under SP at $\rho=1$ across
Manhattan, Chicago, and San Francisco. Colors are standardized within each
panel, so they compare spatial patterns rather than absolute magnitudes.}
\label{fig:app-spatial-maps-sp}
\end{figure}

\begin{figure}[H]
\centering
\includegraphics[width=\linewidth]{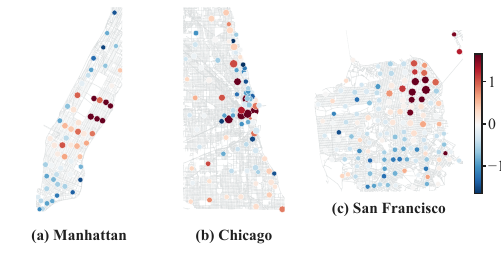}
\Description{Three aligned origin-disparity maps show PIBT in Manhattan,
Chicago, and San Francisco. Colors are within-panel standardized
length-adjusted slowdown.}
\caption{Origin-disparity maps under PIBT at $\rho=1$ across
Manhattan, Chicago, and San Francisco. Colors are standardized within each
panel, so they compare spatial patterns rather than absolute magnitudes.}
\label{fig:app-spatial-maps-pibt}
\end{figure}

\newcommand{\appendixratiotables}{\begin{table}[!t]
\centering
\caption{Cross-city method comparison at $\rho=1/4$. Each city resolves $N=\rho|V|$ on its own graph. Entries are mean $\pm$ population standard deviation over ten fixed demand subsamples; lower is better. \textsc{OOT} and \textsc{INC} denote timeout and incomplete execution.}
\label{tab:app-ratio-1-4}
\label{tab:app-full-matrix}
\begingroup
\scriptsize
\renewcommand{\arraystretch}{0.74}
\setlength{\tabcolsep}{0.75pt}
\let\tablepm\pm
\renewcommand{\pm}{\mathord{\mkern-1mu{\scriptstyle\tablepm}\mkern-1mu}\scriptstyle}
\resizebox{\textwidth}{!}{\begin{tabular}{@{}llrrrrrrrrr@{}}
\toprule
\textbf{City} & \textbf{Method} & \shortstack{Over.\\(\%)} & \shortstack{Mean\\slowdown} & \shortstack{p99\\slowdown} & \shortstack{Origin\\SI} & \shortstack{Destination\\SI} & \shortstack{LI\\10--1} & \shortstack{LI\\8--2} & LGD & \shortstack{CPU\\time (s)} \\
\midrule
\multirow{9}{*}{\rotatebox{90}{Manhattan}} & SP & $66.8\pm4.8$ & $1.55\pm0.04$ & $3.46\pm0.14$ & $0.243\pm0.033$ & $0.101\pm0.019$ & $0.797\pm0.101$ & $0.497\pm0.079$ & $0.282\pm0.032$ & $0.8\pm0.2$ \\
 & GSP & $59.1\pm3.9$ & $1.48\pm0.03$ & $3.31\pm0.14$ & $0.212\pm0.027$ & $0.099\pm0.028$ & $0.751\pm0.080$ & $0.452\pm0.076$ & $0.261\pm0.028$ & \best{0.4\pm0.1} \\
 & TAP & $98.2\pm5.0$ & $1.88\pm0.04$ & $4.49\pm0.28$ & $0.272\pm0.062$ & $0.185\pm0.043$ & $0.690\pm0.118$ & $0.590\pm0.070$ & $0.278\pm0.028$ & $362.5\pm134.4$ \\
 & PIBT & $52.6\pm3.4$ & $1.58\pm0.04$ & $4.46\pm0.38$ & $0.227\pm0.036$ & $0.144\pm0.027$ & $-0.667\pm0.085$ & \best{-0.009\pm0.053} & $0.201\pm0.028$ & $2.7\pm1.0$ \\
 & G-PIBT & \best{46.4\pm1.2} & \best{1.44\pm0.01} & \best{2.65\pm0.12} & \best{0.110\pm0.023} & \best{0.084\pm0.016} & $0.075\pm0.072$ & $0.234\pm0.053$ & \best{0.082\pm0.010} & $321.4\pm101.6$ \\
 & FAIR-RD & $68.0\pm2.7$ & $1.87\pm0.03$ & $6.04\pm0.45$ & $0.279\pm0.060$ & $0.200\pm0.035$ & $-1.830\pm0.172$ & $-0.424\pm0.075$ & $0.475\pm0.045$ & $3.2\pm1.5$ \\
 & GLC & $58.0\pm3.3$ & $1.54\pm0.03$ & $4.14\pm0.28$ & $0.191\pm0.043$ & $0.116\pm0.028$ & $0.161\pm0.060$ & $0.379\pm0.098$ & $0.133\pm0.019$ & $4.8\pm1.7$ \\
 & SPARE ($K=15,R=400$) & $53.4\pm1.9$ & $1.50\pm0.02$ & $3.28\pm0.15$ & $0.163\pm0.035$ & $0.100\pm0.018$ & $0.104\pm0.065$ & $0.299\pm0.045$ & $0.109\pm0.010$ & $7.2\pm3.7$ \\
 & SPARE ($K=5,R=3000$) & $50.7\pm1.5$ & $1.49\pm0.01$ & $2.97\pm0.11$ & $0.144\pm0.026$ & $0.093\pm0.016$ & \best{0.038\pm0.069} & $0.256\pm0.078$ & $0.095\pm0.015$ & $19.8\pm6.1$ \\
\midrule
\multirow{9}{*}{\rotatebox{90}{Chicago}} & SP & $695.7\pm20.6$ & $8.28\pm0.24$ & $26.62\pm1.07$ & $2.431\pm0.114$ & $1.111\pm0.204$ & \best{-0.615\pm0.558} & \best{0.354\pm0.926} & $1.839\pm0.140$ & $11.2\pm2.6$ \\
 & GSP & $620.1\pm16.2$ & $7.23\pm0.14$ & $23.38\pm0.61$ & $2.214\pm0.112$ & $1.073\pm0.156$ & $1.131\pm0.445$ & $1.364\pm0.993$ & $2.042\pm0.114$ & \best{5.9\pm1.8} \\
 & TAP & $783.5\pm78.9$ & $9.73\pm0.69$ & $38.41\pm1.62$ & $2.246\pm0.354$ & $1.010\pm0.131$ & $-2.509\pm1.018$ & $-5.932\pm1.876$ & $2.364\pm0.323$ & $8320.7\pm3205.5$ \\
 & PIBT & $370.1\pm4.8$ & $5.35\pm0.07$ & $21.14\pm0.61$ & \best{1.024\pm0.088} & $0.787\pm0.068$ & $-3.988\pm0.429$ & $-1.590\pm0.515$ & $1.668\pm0.108$ & $57.1\pm12.6$ \\
 & G-PIBT & \textsc{OOT} & \textsc{OOT} & \textsc{OOT} & \textsc{OOT} & \textsc{OOT} & \textsc{OOT} & \textsc{OOT} & \textsc{OOT} & \textsc{OOT} \\
 & FAIR-RD & $402.8\pm8.0$ & $5.96\pm0.10$ & $23.38\pm1.12$ & $1.309\pm0.093$ & $1.063\pm0.059$ & $-5.879\pm0.365$ & $-1.553\pm0.451$ & $2.260\pm0.129$ & $56.1\pm12.9$ \\
 & GLC & $352.3\pm5.4$ & $4.86\pm0.06$ & $18.66\pm0.33$ & $1.180\pm0.083$ & $0.644\pm0.077$ & $-1.472\pm0.268$ & $-0.628\pm0.469$ & \best{1.354\pm0.120} & $54.1\pm17.1$ \\
 & SPARE ($K=15,R=400$) & $349.4\pm5.2$ & $4.84\pm0.06$ & $18.71\pm0.43$ & $1.179\pm0.081$ & $0.630\pm0.071$ & $-1.517\pm0.262$ & $-0.615\pm0.453$ & $1.368\pm0.119$ & $187.9\pm47.5$ \\
 & SPARE ($K=5,R=3000$) & \best{342.3\pm4.9} & \best{4.77\pm0.05} & \best{18.51\pm0.31} & $1.149\pm0.075$ & \best{0.601\pm0.060} & $-1.555\pm0.255$ & $-0.690\pm0.468$ & $1.368\pm0.122$ & $767.1\pm164.3$ \\
\midrule
\multirow{9}{*}{\rotatebox{90}{San Francisco}} & SP & $125.6\pm11.6$ & $2.26\pm0.12$ & $6.99\pm0.93$ & $0.599\pm0.065$ & $0.282\pm0.028$ & $0.447\pm0.118$ & $0.489\pm0.307$ & $0.429\pm0.048$ & $3.6\pm1.2$ \\
 & GSP & $115.6\pm10.7$ & $2.18\pm0.12$ & $7.36\pm0.72$ & $0.560\pm0.061$ & $0.264\pm0.027$ & $0.326\pm0.100$ & $0.339\pm0.223$ & $0.397\pm0.055$ & \best{1.2\pm0.4} \\
 & TAP & $188.9\pm10.2$ & $2.78\pm0.09$ & $8.12\pm0.76$ & $0.582\pm0.053$ & $0.256\pm0.036$ & $0.964\pm0.264$ & $0.962\pm0.354$ & $0.386\pm0.067$ & $1310.9\pm345.8$ \\
 & PIBT & $77.7\pm17.7$ & $2.00\pm0.23$ & $7.84\pm3.11$ & $0.357\pm0.096$ & $0.258\pm0.094$ & $-1.146\pm0.199$ & $-1.011\pm0.373$ & $0.480\pm0.132$ & $15.6\pm5.3$ \\
 & G-PIBT & $83.0\pm3.1$ & $1.86\pm0.04$ & \best{5.09\pm0.22} & \best{0.280\pm0.015} & $0.177\pm0.009$ & $0.197\pm0.072$ & \best{0.066\pm0.144} & \best{0.224\pm0.030} & $5244.1\pm1567.1$ \\
 & FAIR-RD & $89.0\pm17.6$ & $2.29\pm0.25$ & $8.84\pm3.26$ & $0.396\pm0.101$ & $0.327\pm0.096$ & $-2.324\pm0.302$ & $-1.985\pm0.549$ & $0.863\pm0.186$ & $10.7\pm4.9$ \\
 & GLC & $78.5\pm14.8$ & $1.87\pm0.18$ & $6.03\pm1.72$ & $0.333\pm0.080$ & $0.197\pm0.067$ & $-0.050\pm0.104$ & $-0.149\pm0.250$ & $0.321\pm0.108$ & $15.8\pm6.4$ \\
 & SPARE ($K=15,R=400$) & \best{72.4\pm2.4} & \best{1.80\pm0.03} & $5.23\pm0.21$ & $0.299\pm0.015$ & \best{0.173\pm0.012} & \best{-0.042\pm0.084} & $-0.073\pm0.143$ & $0.281\pm0.032$ & $20.8\pm7.6$ \\
 & SPARE ($K=5,R=3000$) & $76.1\pm14.5$ & $1.84\pm0.17$ & $5.68\pm1.72$ & $0.321\pm0.077$ & $0.190\pm0.063$ & $-0.056\pm0.099$ & $-0.165\pm0.248$ & $0.306\pm0.105$ & $87.9\pm24.7$ \\
\bottomrule
\end{tabular}}
\endgroup
\end{table}

\begin{table}[!t]
\centering
\caption{Cross-city method comparison at $\rho=2/4$. Each city resolves $N=\rho|V|$ on its own graph. Entries are mean $\pm$ population standard deviation over ten fixed demand subsamples; lower is better. \textsc{OOT} and \textsc{INC} denote timeout and incomplete execution.}
\label{tab:app-ratio-2-4}
\begingroup
\scriptsize
\renewcommand{\arraystretch}{0.74}
\setlength{\tabcolsep}{0.75pt}
\let\tablepm\pm
\renewcommand{\pm}{\mathord{\mkern-1mu{\scriptstyle\tablepm}\mkern-1mu}\scriptstyle}
\resizebox{\textwidth}{!}{\begin{tabular}{@{}llrrrrrrrrr@{}}
\toprule
\textbf{City} & \textbf{Method} & \shortstack{Over.\\(\%)} & \shortstack{Mean\\slowdown} & \shortstack{p99\\slowdown} & \shortstack{Origin\\SI} & \shortstack{Destination\\SI} & \shortstack{LI\\10--1} & \shortstack{LI\\8--2} & LGD & \shortstack{CPU\\time (s)} \\
\midrule
\multirow{9}{*}{\rotatebox{90}{Manhattan}} & SP & $159.7\pm7.8$ & $2.32\pm0.06$ & $6.17\pm0.30$ & $0.565\pm0.073$ & $0.225\pm0.047$ & $1.863\pm0.219$ & $1.183\pm0.174$ & $0.652\pm0.068$ & $2.0\pm0.5$ \\
 & GSP & $139.5\pm6.7$ & $2.14\pm0.05$ & $5.87\pm0.21$ & $0.488\pm0.056$ & $0.194\pm0.049$ & $1.712\pm0.216$ & $1.054\pm0.164$ & $0.604\pm0.058$ & \best{0.9\pm0.3} \\
 & TAP & $152.7\pm7.8$ & $2.37\pm0.06$ & $6.09\pm0.72$ & $0.473\pm0.075$ & $0.292\pm0.083$ & $1.037\pm0.190$ & $0.849\pm0.200$ & $0.401\pm0.060$ & $646.7\pm234.7$ \\
 & PIBT & $116.3\pm5.5$ & $2.26\pm0.06$ & $7.87\pm0.55$ & $0.524\pm0.107$ & $0.254\pm0.038$ & $-1.173\pm0.185$ & \best{-0.067\pm0.045} & $0.360\pm0.047$ & $8.8\pm1.9$ \\
 & G-PIBT & \best{80.9\pm1.1} & \best{1.78\pm0.01} & \best{3.70\pm0.10} & \best{0.246\pm0.020} & \best{0.131\pm0.015} & $0.075\pm0.088$ & $0.338\pm0.040$ & \best{0.115\pm0.011} & $978.5\pm282.6$ \\
 & FAIR-RD & $144.6\pm4.1$ & $2.79\pm0.05$ & $10.91\pm0.67$ & $0.662\pm0.113$ & $0.387\pm0.056$ & $-3.372\pm0.277$ & $-0.724\pm0.131$ & $0.865\pm0.084$ & $13.7\pm5.8$ \\
 & GLC & $125.3\pm5.0$ & $2.17\pm0.04$ & $7.35\pm0.33$ & $0.483\pm0.065$ & $0.222\pm0.029$ & $0.473\pm0.142$ & $0.715\pm0.142$ & $0.259\pm0.035$ & $14.3\pm6.2$ \\
 & SPARE ($K=15,R=400$) & $103.3\pm1.9$ & $2.00\pm0.02$ & $5.12\pm0.16$ & $0.363\pm0.041$ & $0.175\pm0.020$ & $0.145\pm0.111$ & $0.410\pm0.084$ & $0.148\pm0.016$ & $17.6\pm7.4$ \\
 & SPARE ($K=5,R=3000$) & $90.6\pm2.1$ & $1.88\pm0.02$ & $4.19\pm0.15$ & $0.301\pm0.027$ & $0.145\pm0.019$ & \best{0.043\pm0.110} & $0.366\pm0.087$ & $0.124\pm0.019$ & $55.8\pm27.7$ \\
\midrule
\multirow{9}{*}{\rotatebox{90}{Chicago}} & SP & $1478.1\pm32.2$ & $16.58\pm0.29$ & $52.68\pm1.68$ & $4.998\pm0.129$ & $1.819\pm0.231$ & \best{-1.989\pm1.089} & \best{-0.424\pm2.095} & $3.781\pm0.229$ & $34.5\pm13.0$ \\
 & GSP & $1397.5\pm28.4$ & $14.96\pm0.19$ & $46.18\pm0.52$ & $4.746\pm0.090$ & $2.078\pm0.204$ & $2.891\pm0.971$ & $2.753\pm1.006$ & $3.958\pm0.233$ & \best{26.7\pm10.8} \\
 & TAP & \textsc{OOT} & \textsc{OOT} & \textsc{OOT} & \textsc{OOT} & \textsc{OOT} & \textsc{OOT} & \textsc{OOT} & \textsc{OOT} & \textsc{OOT} \\
 & PIBT & $753.4\pm12.0$ & $9.84\pm0.17$ & $41.10\pm1.18$ & \best{1.967\pm0.111} & $1.524\pm0.106$ & $-7.950\pm0.497$ & $-3.700\pm0.525$ & $3.347\pm0.117$ & $322.6\pm101.6$ \\
 & G-PIBT & \textsc{OOT} & \textsc{OOT} & \textsc{OOT} & \textsc{OOT} & \textsc{OOT} & \textsc{OOT} & \textsc{OOT} & \textsc{OOT} & \textsc{OOT} \\
 & FAIR-RD & $806.1\pm10.4$ & $10.76\pm0.12$ & $42.96\pm1.75$ & $2.572\pm0.138$ & $1.996\pm0.111$ & $-10.631\pm0.469$ & $-3.009\pm0.642$ & $4.207\pm0.134$ & $232.5\pm88.0$ \\
 & GLC & $716.5\pm9.8$ & $8.92\pm0.12$ & $36.72\pm1.01$ & $2.265\pm0.107$ & $1.226\pm0.104$ & $-3.450\pm0.353$ & $-2.094\pm0.613$ & $2.821\pm0.130$ & $232.0\pm78.6$ \\
 & SPARE ($K=15,R=400$) & $708.4\pm9.6$ & $8.85\pm0.12$ & $36.74\pm1.01$ & $2.195\pm0.111$ & $1.186\pm0.097$ & $-3.469\pm0.366$ & $-2.234\pm0.574$ & \best{2.812\pm0.131} & $568.7\pm212.8$ \\
 & SPARE ($K=5,R=3000$) & \best{695.5\pm9.3} & \best{8.75\pm0.10} & \best{36.59\pm0.93} & $2.114\pm0.130$ & \best{1.165\pm0.089} & $-3.623\pm0.312$ & $-2.255\pm0.618$ & $2.825\pm0.130$ & $2578.2\pm620.7$ \\
\midrule
\multirow{9}{*}{\rotatebox{90}{San Francisco}} & SP & $306.9\pm12.1$ & $4.07\pm0.13$ & $14.74\pm1.14$ & $1.373\pm0.070$ & $0.665\pm0.042$ & $1.178\pm0.129$ & $0.866\pm0.372$ & $0.919\pm0.050$ & $7.2\pm2.7$ \\
 & GSP & $286.8\pm11.3$ & $3.93\pm0.13$ & $14.66\pm0.84$ & $1.320\pm0.051$ & $0.611\pm0.051$ & $0.916\pm0.123$ & $0.734\pm0.389$ & $0.919\pm0.066$ & \best{3.5\pm1.3} \\
 & TAP & $338.1\pm23.9$ & $4.40\pm0.26$ & $14.76\pm1.77$ & $1.267\pm0.126$ & $0.504\pm0.056$ & $1.072\pm0.522$ & $0.980\pm0.652$ & $0.840\pm0.141$ & $3411.2\pm707.6$ \\
 & PIBT & $152.4\pm2.3$ & $2.93\pm0.03$ & $12.53\pm0.31$ & $0.637\pm0.015$ & $0.461\pm0.014$ & $-2.197\pm0.146$ & $-1.710\pm0.170$ & $0.877\pm0.043$ & $48.7\pm15.1$ \\
 & G-PIBT & \textsc{OOT} & \textsc{OOT} & \textsc{OOT} & \textsc{OOT} & \textsc{OOT} & \textsc{OOT} & \textsc{OOT} & \textsc{OOT} & \textsc{OOT} \\
 & FAIR-RD & $169.3\pm2.5$ & $3.38\pm0.04$ & $14.14\pm0.29$ & $0.678\pm0.027$ & $0.558\pm0.020$ & $-3.987\pm0.162$ & $-3.020\pm0.178$ & $1.444\pm0.055$ & $28.9\pm11.9$ \\
 & GLC & $150.7\pm3.0$ & $2.71\pm0.04$ & $9.82\pm0.30$ & $0.616\pm0.019$ & $0.387\pm0.021$ & \best{-0.228\pm0.106} & $-0.582\pm0.167$ & $0.636\pm0.037$ & $34.8\pm10.9$ \\
 & SPARE ($K=15,R=400$) & $147.5\pm2.9$ & $2.67\pm0.04$ & $9.52\pm0.25$ & $0.601\pm0.019$ & $0.375\pm0.017$ & $-0.237\pm0.108$ & $-0.614\pm0.156$ & $0.617\pm0.040$ & $69.2\pm20.0$ \\
 & SPARE ($K=5,R=3000$) & \best{142.1\pm2.6} & \best{2.61\pm0.04} & \best{9.26\pm0.28} & \best{0.580\pm0.018} & \best{0.357\pm0.017} & $-0.276\pm0.135$ & \best{-0.563\pm0.147} & \best{0.600\pm0.033} & $243.4\pm80.4$ \\
\bottomrule
\end{tabular}}
\endgroup
\end{table}

\begin{table}[!t]
\centering
\caption{Cross-city method comparison at $\rho=3/4$. Each city resolves $N=\rho|V|$ on its own graph. Entries are mean $\pm$ population standard deviation over ten fixed demand subsamples; lower is better. \textsc{OOT} and \textsc{INC} denote timeout and incomplete execution.}
\label{tab:app-ratio-3-4}
\begingroup
\scriptsize
\renewcommand{\arraystretch}{0.74}
\setlength{\tabcolsep}{0.75pt}
\let\tablepm\pm
\renewcommand{\pm}{\mathord{\mkern-1mu{\scriptstyle\tablepm}\mkern-1mu}\scriptstyle}
\resizebox{\textwidth}{!}{\begin{tabular}{@{}llrrrrrrrrr@{}}
\toprule
\textbf{City} & \textbf{Method} & \shortstack{Over.\\(\%)} & \shortstack{Mean\\slowdown} & \shortstack{p99\\slowdown} & \shortstack{Origin\\SI} & \shortstack{Destination\\SI} & \shortstack{LI\\10--1} & \shortstack{LI\\8--2} & LGD & \shortstack{CPU\\time (s)} \\
\midrule
\multirow{9}{*}{\rotatebox{90}{Manhattan}} & SP & $260.4\pm6.8$ & $3.17\pm0.05$ & $9.14\pm0.36$ & $0.943\pm0.060$ & $0.314\pm0.066$ & $2.945\pm0.195$ & $1.840\pm0.183$ & $1.019\pm0.063$ & $3.2\pm0.4$ \\
 & GSP & $225.7\pm6.7$ & $2.86\pm0.05$ & $8.52\pm0.35$ & $0.756\pm0.049$ & $0.275\pm0.063$ & $2.694\pm0.196$ & $1.648\pm0.206$ & $0.946\pm0.066$ & \best{1.9\pm0.4} \\
 & TAP & $213.3\pm15.3$ & $2.91\pm0.11$ & $7.40\pm0.70$ & $0.698\pm0.092$ & $0.332\pm0.127$ & $1.545\pm0.347$ & $1.292\pm0.249$ & $0.576\pm0.115$ & $604.2\pm123.8$ \\
 & PIBT & $183.3\pm6.1$ & $2.97\pm0.06$ & $12.01\pm0.52$ & $0.780\pm0.090$ & $0.408\pm0.049$ & $-1.777\pm0.226$ & \best{-0.004\pm0.144} & $0.534\pm0.063$ & $24.9\pm4.6$ \\
 & G-PIBT & \best{113.7\pm0.7} & \best{2.12\pm0.01} & \best{4.68\pm0.07} & \best{0.348\pm0.030} & \best{0.168\pm0.017} & $0.071\pm0.062$ & $0.343\pm0.054$ & $0.130\pm0.014$ & $2110.9\pm796.8$ \\
 & FAIR-RD & $221.4\pm6.4$ & $3.68\pm0.06$ & $15.23\pm0.81$ & $1.004\pm0.128$ & $0.545\pm0.064$ & $-4.617\pm0.189$ & $-1.059\pm0.193$ & $1.172\pm0.051$ & $27.9\pm8.2$ \\
 & GLC & $192.8\pm5.9$ & $2.82\pm0.05$ & $10.63\pm0.50$ & $0.736\pm0.076$ & $0.348\pm0.046$ & $0.694\pm0.140$ & $0.926\pm0.175$ & $0.347\pm0.038$ & $22.5\pm5.6$ \\
 & SPARE ($K=15,R=400$) & $153.2\pm2.2$ & $2.51\pm0.02$ & $7.30\pm0.27$ & $0.495\pm0.062$ & $0.238\pm0.033$ & $0.147\pm0.087$ & $0.408\pm0.145$ & $0.152\pm0.029$ & $37.5\pm13.6$ \\
 & SPARE ($K=5,R=3000$) & $127.2\pm1.9$ & $2.26\pm0.02$ & $5.59\pm0.13$ & $0.396\pm0.034$ & $0.187\pm0.023$ & \best{-0.014\pm0.066} & $0.305\pm0.108$ & \best{0.118\pm0.022} & $120.3\pm44.3$ \\
\midrule
\multirow{9}{*}{\rotatebox{90}{Chicago}} & SP & $2292.3\pm64.3$ & $25.07\pm0.71$ & $78.42\pm2.97$ & $7.707\pm0.298$ & $2.332\pm0.362$ & \best{-2.481\pm1.339} & \best{-1.456\pm2.220} & $5.409\pm0.248$ & \best{53.9\pm21.8} \\
 & GSP & $2246.2\pm46.4$ & $23.28\pm0.45$ & $69.31\pm1.66$ & $7.288\pm0.231$ & $3.005\pm0.451$ & $5.053\pm1.181$ & $4.908\pm0.663$ & $6.061\pm0.289$ & $55.7\pm19.9$ \\
 & TAP & \textsc{OOT} & \textsc{OOT} & \textsc{OOT} & \textsc{OOT} & \textsc{OOT} & \textsc{OOT} & \textsc{OOT} & \textsc{OOT} & \textsc{OOT} \\
 & PIBT & $1141.1\pm12.5$ & $14.41\pm0.16$ & $61.16\pm1.75$ & \best{2.840\pm0.100} & $2.266\pm0.124$ & $-12.261\pm0.740$ & $-5.737\pm0.370$ & $5.019\pm0.144$ & $867.9\pm291.5$ \\
 & G-PIBT & \textsc{OOT} & \textsc{OOT} & \textsc{OOT} & \textsc{OOT} & \textsc{OOT} & \textsc{OOT} & \textsc{OOT} & \textsc{OOT} & \textsc{OOT} \\
 & FAIR-RD & $1211.2\pm11.6$ & $15.59\pm0.18$ & $63.99\pm2.43$ & $3.746\pm0.151$ & $2.949\pm0.158$ & $-15.365\pm0.725$ & $-4.695\pm0.707$ & $5.908\pm0.198$ & $455.7\pm150.0$ \\
 & GLC & $1087.1\pm9.8$ & $13.08\pm0.12$ & \best{54.46\pm1.36} & $3.379\pm0.104$ & $1.895\pm0.155$ & $-5.770\pm0.522$ & $-3.413\pm0.432$ & $4.177\pm0.155$ & $665.6\pm194.8$ \\
 & SPARE ($K=15,R=400$) & $1071.6\pm10.7$ & $12.93\pm0.13$ & $54.72\pm1.52$ & $3.188\pm0.148$ & $1.829\pm0.137$ & $-5.837\pm0.497$ & $-3.537\pm0.485$ & \best{4.162\pm0.160} & $1165.6\pm347.1$ \\
 & SPARE ($K=5,R=3000$) & \best{1061.8\pm11.3} & \best{12.86\pm0.15} & $55.10\pm1.37$ & $3.125\pm0.095$ & \best{1.801\pm0.131} & $-5.951\pm0.605$ & $-3.569\pm0.481$ & $4.169\pm0.170$ & $6050.2\pm1367.8$ \\
\midrule
\multirow{9}{*}{\rotatebox{90}{San Francisco}} & SP & $505.0\pm17.0$ & $6.04\pm0.18$ & $21.99\pm1.22$ & $2.188\pm0.107$ & $1.094\pm0.059$ & $1.911\pm0.198$ & $1.524\pm0.349$ & $1.435\pm0.102$ & $14.0\pm3.6$ \\
 & GSP & $472.4\pm14.4$ & $5.79\pm0.17$ & $21.81\pm1.24$ & $2.114\pm0.083$ & $0.984\pm0.050$ & $1.561\pm0.172$ & $1.420\pm0.488$ & $1.450\pm0.114$ & \best{6.8\pm1.9} \\
 & TAP & $493.4\pm36.7$ & $6.07\pm0.39$ & $22.34\pm2.06$ & $2.012\pm0.211$ & $0.859\pm0.123$ & $1.229\pm0.504$ & \best{0.650\pm0.978} & $1.323\pm0.222$ & $4505.8\pm1360.5$ \\
 & PIBT & $235.9\pm3.6$ & $3.97\pm0.05$ & $18.38\pm0.43$ & $0.947\pm0.027$ & $0.706\pm0.032$ & $-3.423\pm0.213$ & $-2.361\pm0.177$ & $1.312\pm0.055$ & $114.4\pm33.1$ \\
 & G-PIBT & \textsc{OOT} & \textsc{OOT} & \textsc{OOT} & \textsc{OOT} & \textsc{OOT} & \textsc{OOT} & \textsc{OOT} & \textsc{OOT} & \textsc{OOT} \\
 & FAIR-RD & $254.8\pm2.5$ & $4.53\pm0.04$ & $20.58\pm0.26$ & $0.984\pm0.034$ & $0.810\pm0.025$ & $-5.704\pm0.257$ & $-4.066\pm0.159$ & $2.028\pm0.059$ & $69.7\pm24.9$ \\
 & GLC & $230.8\pm3.2$ & $3.66\pm0.04$ & $14.44\pm0.36$ & $0.937\pm0.022$ & $0.619\pm0.024$ & \best{-0.685\pm0.169} & $-1.178\pm0.128$ & $0.978\pm0.036$ & $85.1\pm26.3$ \\
 & SPARE ($K=15,R=400$) & $225.7\pm2.9$ & $3.60\pm0.04$ & $14.08\pm0.43$ & $0.911\pm0.021$ & $0.600\pm0.023$ & $-0.707\pm0.142$ & $-1.175\pm0.115$ & $0.949\pm0.043$ & $131.8\pm43.4$ \\
 & SPARE ($K=5,R=3000$) & \best{217.1\pm3.0} & \best{3.51\pm0.04} & \best{13.88\pm0.48} & \best{0.881\pm0.022} & \best{0.573\pm0.023} & $-0.740\pm0.166$ & $-1.136\pm0.131$ & \best{0.927\pm0.034} & $509.3\pm142.8$ \\
\bottomrule
\end{tabular}}
\endgroup
\end{table}

\begin{table}[!t]
\centering
\caption{Cross-city method comparison at $\rho=1$. Each city resolves $N=\rho|V|$ on its own graph. Entries are mean $\pm$ population standard deviation over ten fixed demand subsamples; lower is better. \textsc{OOT} and \textsc{INC} denote timeout and incomplete execution.}
\label{tab:app-ratio-1}
\label{tab:app-full-matrix-3density}
\begingroup
\scriptsize
\renewcommand{\arraystretch}{0.74}
\setlength{\tabcolsep}{0.75pt}
\let\tablepm\pm
\renewcommand{\pm}{\mathord{\mkern-1mu{\scriptstyle\tablepm}\mkern-1mu}\scriptstyle}
\resizebox{\textwidth}{!}{\begin{tabular}{@{}llrrrrrrrrr@{}}
\toprule
\textbf{City} & \textbf{Method} & \shortstack{Over.\\(\%)} & \shortstack{Mean\\slowdown} & \shortstack{p99\\slowdown} & \shortstack{Origin\\SI} & \shortstack{Destination\\SI} & \shortstack{LI\\10--1} & \shortstack{LI\\8--2} & LGD & \shortstack{CPU\\time (s)} \\
\midrule
\multirow{9}{*}{\rotatebox{90}{Manhattan}} & SP & $367.9\pm8.3$ & $4.08\pm0.07$ & $12.03\pm0.39$ & $1.386\pm0.126$ & $0.442\pm0.074$ & $4.059\pm0.175$ & $2.548\pm0.156$ & $1.403\pm0.065$ & $4.0\pm0.9$ \\
 & GSP & $324.9\pm6.3$ & $3.70\pm0.06$ & $11.41\pm0.35$ & $1.112\pm0.039$ & $0.372\pm0.042$ & $3.763\pm0.211$ & $2.356\pm0.138$ & $1.298\pm0.073$ & \best{3.7\pm0.5} \\
 & TAP & $289.8\pm23.4$ & $3.61\pm0.21$ & $10.09\pm1.26$ & $0.973\pm0.112$ & $0.471\pm0.154$ & $1.894\pm0.171$ & $1.775\pm0.299$ & $0.746\pm0.090$ & $974.3\pm257.1$ \\
 & PIBT & $251.6\pm7.4$ & $3.69\pm0.06$ & $15.84\pm0.76$ & $1.047\pm0.078$ & $0.558\pm0.031$ & $-2.305\pm0.173$ & \best{-0.036\pm0.163} & $0.699\pm0.050$ & $57.1\pm22.8$ \\
 & G-PIBT & \best{146.0\pm0.5} & \best{2.46\pm0.01} & \best{5.74\pm0.16} & \best{0.459\pm0.027} & \best{0.220\pm0.017} & \best{-0.029\pm0.062} & $0.275\pm0.068$ & $0.134\pm0.017$ & $4034.6\pm1252.8$ \\
 & FAIR-RD & $292.6\pm6.2$ & $4.56\pm0.07$ & $19.84\pm1.13$ & $1.316\pm0.096$ & $0.761\pm0.078$ & $-6.235\pm0.210$ & $-1.379\pm0.176$ & $1.588\pm0.057$ & $41.6\pm18.8$ \\
 & GLC & $256.5\pm7.4$ & $3.45\pm0.07$ & $13.42\pm0.71$ & $1.013\pm0.078$ & $0.487\pm0.045$ & $0.843\pm0.111$ & $0.856\pm0.110$ & $0.381\pm0.037$ & $40.4\pm15.6$ \\
 & SPARE ($K=15,R=400$) & $203.6\pm3.5$ & $3.03\pm0.03$ & $9.41\pm0.24$ & $0.675\pm0.042$ & $0.329\pm0.021$ & $0.129\pm0.070$ & $0.212\pm0.079$ & $0.155\pm0.025$ & $53.2\pm22.6$ \\
 & SPARE ($K=5,R=3000$) & $165.0\pm2.4$ & $2.67\pm0.02$ & $7.06\pm0.17$ & $0.500\pm0.035$ & $0.229\pm0.017$ & $-0.162\pm0.073$ & $0.053\pm0.099$ & \best{0.098\pm0.018} & $158.8\pm62.1$ \\
\midrule
\multirow{9}{*}{\rotatebox{90}{Chicago}} & SP & $3222.5\pm106.1$ & $34.36\pm1.06$ & $103.62\pm3.89$ & $10.323\pm0.280$ & $2.828\pm0.383$ & \best{-1.356\pm1.349} & \best{-0.132\pm2.432} & $7.080\pm0.355$ & $115.6\pm46.5$ \\
 & GSP & $3153.4\pm58.7$ & $31.84\pm0.57$ & $91.20\pm1.59$ & $9.996\pm0.220$ & $3.550\pm0.314$ & $8.981\pm0.599$ & $7.207\pm0.967$ & $7.888\pm0.384$ & \best{99.6\pm43.7} \\
 & TAP & \textsc{OOT} & \textsc{OOT} & \textsc{OOT} & \textsc{OOT} & \textsc{OOT} & \textsc{OOT} & \textsc{OOT} & \textsc{OOT} & \textsc{OOT} \\
 & PIBT & $1540.8\pm14.6$ & $19.17\pm0.21$ & $81.29\pm1.13$ & \best{3.825\pm0.185} & $2.931\pm0.121$ & $-16.860\pm0.690$ & $-7.256\pm0.578$ & $6.735\pm0.135$ & $1811.9\pm834.5$ \\
 & G-PIBT & \textsc{OOT} & \textsc{OOT} & \textsc{OOT} & \textsc{OOT} & \textsc{OOT} & \textsc{OOT} & \textsc{OOT} & \textsc{OOT} & \textsc{OOT} \\
 & FAIR-RD & $1622.2\pm8.3$ & $20.51\pm0.15$ & $84.29\pm1.87$ & $5.054\pm0.238$ & $3.833\pm0.128$ & $-20.294\pm0.894$ & $-5.995\pm0.996$ & $7.858\pm0.113$ & $814.5\pm303.0$ \\
 & GLC & $1463.4\pm10.9$ & $17.31\pm0.13$ & $72.37\pm1.63$ & $4.540\pm0.224$ & $2.420\pm0.104$ & $-8.257\pm0.614$ & $-4.334\pm0.588$ & $5.672\pm0.130$ & $1073.4\pm412.2$ \\
 & SPARE ($K=15,R=400$) & $1440.0\pm13.0$ & $17.07\pm0.16$ & $72.33\pm1.59$ & $4.248\pm0.184$ & $2.359\pm0.089$ & $-8.123\pm0.658$ & $-4.434\pm0.590$ & \best{5.638\pm0.144} & $1599.7\pm581.2$ \\
 & SPARE ($K=5,R=3000$) & \best{1428.3\pm9.3} & \best{16.96\pm0.13} & \best{72.27\pm1.11} & $4.216\pm0.191$ & \best{2.330\pm0.094} & $-8.242\pm0.588$ & $-4.243\pm0.649$ & $5.647\pm0.153$ & $8513.4\pm2383.5$ \\
\midrule
\multirow{9}{*}{\rotatebox{90}{San Francisco}} & SP & $727.2\pm24.9$ & $8.23\pm0.27$ & $29.04\pm0.98$ & $3.087\pm0.155$ & $1.547\pm0.068$ & $2.934\pm0.282$ & $2.206\pm0.427$ & $2.031\pm0.149$ & $20.9\pm6.9$ \\
 & GSP & $662.2\pm14.5$ & $7.69\pm0.17$ & $29.04\pm0.89$ & $2.905\pm0.077$ & $1.345\pm0.057$ & $2.425\pm0.238$ & $1.853\pm0.474$ & $2.044\pm0.143$ & \best{13.2\pm5.0} \\
 & TAP & $676.7\pm33.8$ & $8.07\pm0.41$ & $29.84\pm2.39$ & $2.839\pm0.180$ & $1.253\pm0.118$ & \best{0.698\pm0.956} & \best{0.334\pm0.900} & $1.739\pm0.316$ & $7355.5\pm1756.8$ \\
 & PIBT & $318.6\pm4.1$ & $5.02\pm0.06$ & $24.20\pm0.56$ & $1.250\pm0.030$ & $0.933\pm0.023$ & $-4.844\pm0.286$ & $-3.124\pm0.275$ & $1.801\pm0.070$ & $264.7\pm75.8$ \\
 & G-PIBT & \textsc{OOT} & \textsc{OOT} & \textsc{OOT} & \textsc{OOT} & \textsc{OOT} & \textsc{OOT} & \textsc{OOT} & \textsc{OOT} & \textsc{OOT} \\
 & FAIR-RD & $341.2\pm3.8$ & $5.68\pm0.06$ & $26.89\pm0.44$ & $1.275\pm0.045$ & $1.059\pm0.029$ & $-7.497\pm0.328$ & $-5.133\pm0.265$ & $2.656\pm0.079$ & $101.9\pm35.7$ \\
 & GLC & $310.3\pm3.7$ & $4.60\pm0.05$ & $18.93\pm0.41$ & $1.233\pm0.029$ & $0.837\pm0.023$ & $-1.186\pm0.198$ & $-1.792\pm0.193$ & $1.343\pm0.040$ & $139.6\pm46.0$ \\
 & SPARE ($K=15,R=400$) & $303.2\pm3.3$ & $4.51\pm0.05$ & $18.38\pm0.43$ & $1.199\pm0.027$ & $0.806\pm0.024$ & $-1.178\pm0.211$ & $-1.788\pm0.179$ & $1.298\pm0.037$ & $182.1\pm64.7$ \\
 & SPARE ($K=5,R=3000$) & \best{293.8\pm3.1} & \best{4.42\pm0.05} & \best{18.10\pm0.54} & \best{1.162\pm0.023} & \best{0.779\pm0.024} & $-1.246\pm0.176$ & $-1.798\pm0.190$ & \best{1.272\pm0.038} & $705.9\pm232.7$ \\
\bottomrule
\end{tabular}}
\endgroup
\end{table}

\begin{table}[!t]
\centering
\caption{Cross-city method comparison at $\rho=5/4$. Each city resolves $N=\rho|V|$ on its own graph. Entries are mean $\pm$ population standard deviation over ten fixed demand subsamples; lower is better. \textsc{OOT} and \textsc{INC} denote timeout and incomplete execution.}
\label{tab:app-ratio-5-4}
\begingroup
\scriptsize
\renewcommand{\arraystretch}{0.74}
\setlength{\tabcolsep}{0.75pt}
\let\tablepm\pm
\renewcommand{\pm}{\mathord{\mkern-1mu{\scriptstyle\tablepm}\mkern-1mu}\scriptstyle}
\resizebox{\textwidth}{!}{\begin{tabular}{@{}llrrrrrrrrr@{}}
\toprule
\textbf{City} & \textbf{Method} & \shortstack{Over.\\(\%)} & \shortstack{Mean\\slowdown} & \shortstack{p99\\slowdown} & \shortstack{Origin\\SI} & \shortstack{Destination\\SI} & \shortstack{LI\\10--1} & \shortstack{LI\\8--2} & LGD & \shortstack{CPU\\time (s)} \\
\midrule
\multirow{9}{*}{\rotatebox{90}{Manhattan}} & SP & $480.0\pm16.3$ & $5.04\pm0.14$ & $14.86\pm0.45$ & $1.831\pm0.148$ & $0.592\pm0.129$ & $5.116\pm0.200$ & $3.319\pm0.260$ & $1.771\pm0.075$ & $6.2\pm0.4$ \\
 & GSP & $424.4\pm7.9$ & $4.53\pm0.06$ & $14.16\pm0.32$ & $1.460\pm0.057$ & $0.434\pm0.057$ & $4.800\pm0.204$ & $3.059\pm0.161$ & $1.666\pm0.079$ & \best{4.9\pm1.0} \\
 & TAP & $346.1\pm20.1$ & $4.12\pm0.18$ & $11.57\pm1.36$ & $1.272\pm0.117$ & $0.440\pm0.161$ & $2.256\pm0.322$ & $1.997\pm0.324$ & $0.894\pm0.104$ & $1481.0\pm512.8$ \\
 & PIBT & $316.5\pm6.9$ & $4.41\pm0.06$ & $19.67\pm0.72$ & $1.318\pm0.094$ & $0.712\pm0.058$ & $-3.105\pm0.169$ & \best{-0.060\pm0.133} & $0.922\pm0.040$ & $85.4\pm35.1$ \\
 & G-PIBT & \best{177.6\pm1.0} & \best{2.80\pm0.01} & \best{6.83\pm0.18} & \best{0.562\pm0.030} & \best{0.265\pm0.018} & $-0.164\pm0.065$ & $0.213\pm0.057$ & $0.156\pm0.017$ & $6604.6\pm2046.2$ \\
 & FAIR-RD & $366.8\pm7.6$ & $5.44\pm0.08$ & $24.49\pm0.91$ & $1.683\pm0.100$ & $0.984\pm0.071$ & $-7.775\pm0.321$ & $-1.683\pm0.197$ & $1.982\pm0.086$ & $66.1\pm28.2$ \\
 & GLC & $317.4\pm5.9$ & $4.06\pm0.06$ & $16.31\pm0.57$ & $1.266\pm0.061$ & $0.634\pm0.046$ & $0.807\pm0.152$ & $0.920\pm0.189$ & $0.414\pm0.030$ & $70.6\pm28.2$ \\
 & SPARE ($K=15,R=400$) & $256.1\pm3.5$ & $3.58\pm0.04$ & $11.62\pm0.48$ & $0.858\pm0.043$ & $0.413\pm0.036$ & \best{-0.041\pm0.152} & $0.166\pm0.103$ & $0.177\pm0.025$ & $70.6\pm30.5$ \\
 & SPARE ($K=5,R=3000$) & $203.4\pm2.0$ & $3.11\pm0.02$ & $8.97\pm0.17$ & $0.643\pm0.034$ & $0.274\pm0.023$ & $-0.433\pm0.091$ & $-0.211\pm0.057$ & \best{0.156\pm0.015} & $277.2\pm96.7$ \\
\midrule
\multirow{9}{*}{\rotatebox{90}{Chicago}} & SP & $4281.3\pm123.6$ & $44.68\pm1.04$ & $129.54\pm3.01$ & $13.233\pm0.354$ & $3.812\pm0.482$ & \best{2.217\pm2.171} & \best{2.168\pm3.279} & $8.882\pm0.404$ & $169.1\pm59.9$ \\
 & GSP & $4052.0\pm45.8$ & $40.53\pm0.45$ & $113.50\pm1.57$ & $12.769\pm0.226$ & $4.379\pm0.294$ & $12.538\pm1.346$ & $8.990\pm0.781$ & $9.924\pm0.318$ & \best{158.9\pm69.8} \\
 & TAP & \textsc{OOT} & \textsc{OOT} & \textsc{OOT} & \textsc{OOT} & \textsc{OOT} & \textsc{OOT} & \textsc{OOT} & \textsc{OOT} & \textsc{OOT} \\
 & PIBT & $1959.9\pm13.2$ & $24.25\pm0.19$ & $102.14\pm1.26$ & \best{4.805\pm0.166} & $3.636\pm0.096$ & $-22.109\pm0.751$ & $-9.836\pm0.768$ & $8.426\pm0.075$ & $4811.8\pm1812.9$ \\
 & G-PIBT & \textsc{OOT} & \textsc{OOT} & \textsc{OOT} & \textsc{OOT} & \textsc{OOT} & \textsc{OOT} & \textsc{OOT} & \textsc{OOT} & \textsc{OOT} \\
 & FAIR-RD & $2153.0\pm115.9$ & $27.07\pm1.64$ & $109.82\pm5.18$ & $6.993\pm0.675$ & $5.174\pm0.533$ & $-28.843\pm4.356$ & $-8.197\pm1.889$ & $10.102\pm0.537$ & $1050.4\pm196.0$ \\
 & GLC & $1833.2\pm11.9$ & $21.56\pm0.18$ & $90.89\pm1.93$ & $5.585\pm0.135$ & $3.064\pm0.132$ & $-11.082\pm0.777$ & $-5.803\pm1.175$ & $6.970\pm0.143$ & $1352.5\pm513.1$ \\
 & SPARE ($K=15,R=400$) & $1809.6\pm13.4$ & $21.28\pm0.16$ & \best{90.49\pm1.50} & $5.320\pm0.121$ & $2.961\pm0.103$ & $-10.715\pm0.575$ & $-5.875\pm0.919$ & \best{6.912\pm0.166} & $2486.7\pm869.7$ \\
 & SPARE ($K=5,R=3000$) & \best{1801.4\pm12.9} & \best{21.21\pm0.18} & $90.59\pm1.55$ & $5.304\pm0.100$ & \best{2.947\pm0.118} & $-10.789\pm0.774$ & $-5.799\pm1.117$ & $6.940\pm0.142$ & $10452.2\pm1884.7$ \\
\midrule
\multirow{9}{*}{\rotatebox{90}{San Francisco}} & SP & $948.4\pm25.9$ & $10.40\pm0.28$ & $36.21\pm1.78$ & $3.924\pm0.163$ & $2.002\pm0.082$ & $3.910\pm0.226$ & $2.887\pm0.419$ & $2.597\pm0.164$ & $32.6\pm10.2$ \\
 & GSP & $856.9\pm19.2$ & $9.66\pm0.22$ & $36.36\pm1.12$ & $3.689\pm0.102$ & $1.712\pm0.042$ & $3.317\pm0.213$ & $2.410\pm0.451$ & $2.668\pm0.135$ & \best{20.0\pm7.7} \\
 & TAP & $900.4\pm87.5$ & $10.40\pm0.97$ & $37.59\pm5.30$ & $3.832\pm0.449$ & $1.688\pm0.260$ & \best{1.706\pm0.705} & \best{1.263\pm1.351} & $2.794\pm0.380$ & $8649.0\pm2876.1$ \\
 & PIBT & $402.8\pm3.9$ & $6.08\pm0.05$ & $29.78\pm0.72$ & $1.551\pm0.034$ & $1.170\pm0.020$ & $-6.112\pm0.326$ & $-3.923\pm0.244$ & $2.273\pm0.064$ & $366.2\pm104.5$ \\
 & G-PIBT & \textsc{OOT} & \textsc{OOT} & \textsc{OOT} & \textsc{OOT} & \textsc{OOT} & \textsc{OOT} & \textsc{OOT} & \textsc{OOT} & \textsc{OOT} \\
 & FAIR-RD & $427.1\pm4.5$ & $6.82\pm0.07$ & $33.43\pm0.52$ & $1.575\pm0.054$ & $1.310\pm0.033$ & $-9.075\pm0.315$ & $-6.242\pm0.283$ & $3.244\pm0.079$ & $143.8\pm43.8$ \\
 & GLC & $389.7\pm4.4$ & $5.56\pm0.06$ & $23.45\pm0.43$ & $1.536\pm0.035$ & $1.062\pm0.018$ & $-1.863\pm0.224$ & $-2.512\pm0.192$ & $1.713\pm0.044$ & $209.3\pm62.5$ \\
 & SPARE ($K=15,R=400$) & $381.8\pm4.1$ & $5.47\pm0.06$ & $22.69\pm0.53$ & $1.496\pm0.032$ & $1.030\pm0.017$ & $-1.828\pm0.250$ & $-2.515\pm0.195$ & $1.669\pm0.037$ & $211.9\pm71.1$ \\
 & SPARE ($K=5,R=3000$) & \best{371.2\pm3.5} & \best{5.35\pm0.05} & \best{22.61\pm0.59} & \best{1.452\pm0.029} & \best{0.995\pm0.020} & $-1.869\pm0.207$ & $-2.480\pm0.179$ & \best{1.623\pm0.042} & $784.1\pm259.4$ \\
\bottomrule
\end{tabular}}
\endgroup
\end{table}

\begin{table}[!t]
\centering
\caption{Cross-city method comparison at $\rho=6/4$. Each city resolves $N=\rho|V|$ on its own graph. Entries are mean $\pm$ population standard deviation over ten fixed demand subsamples; lower is better. \textsc{OOT} and \textsc{INC} denote timeout and incomplete execution.}
\label{tab:app-ratio-6-4}
\begingroup
\scriptsize
\renewcommand{\arraystretch}{0.74}
\setlength{\tabcolsep}{0.75pt}
\let\tablepm\pm
\renewcommand{\pm}{\mathord{\mkern-1mu{\scriptstyle\tablepm}\mkern-1mu}\scriptstyle}
\resizebox{\textwidth}{!}{\begin{tabular}{@{}llrrrrrrrrr@{}}
\toprule
\textbf{City} & \textbf{Method} & \shortstack{Over.\\(\%)} & \shortstack{Mean\\slowdown} & \shortstack{p99\\slowdown} & \shortstack{Origin\\SI} & \shortstack{Destination\\SI} & \shortstack{LI\\10--1} & \shortstack{LI\\8--2} & LGD & \shortstack{CPU\\time (s)} \\
\midrule
\multirow{9}{*}{\rotatebox{90}{Manhattan}} & SP & $590.9\pm16.9$ & $6.00\pm0.14$ & $17.73\pm0.32$ & $2.244\pm0.138$ & $0.741\pm0.122$ & $6.161\pm0.287$ & $4.056\pm0.214$ & $2.135\pm0.075$ & $7.0\pm1.8$ \\
 & GSP & $527.3\pm7.4$ & $5.40\pm0.06$ & $16.91\pm0.36$ & $1.854\pm0.053$ & $0.545\pm0.057$ & $5.836\pm0.205$ & $3.865\pm0.155$ & $2.050\pm0.059$ & \best{5.9\pm1.5} \\
 & TAP & $446.3\pm32.8$ & $5.04\pm0.33$ & $15.02\pm2.13$ & $1.693\pm0.229$ & $0.642\pm0.219$ & $2.883\pm0.295$ & $2.443\pm0.448$ & $1.098\pm0.114$ & $1299.8\pm242.1$ \\
 & PIBT & $379.1\pm6.3$ & $5.09\pm0.07$ & $23.23\pm0.70$ & $1.623\pm0.140$ & $0.869\pm0.064$ & $-3.873\pm0.255$ & \best{-0.120\pm0.192} & $1.128\pm0.083$ & $120.3\pm50.8$ \\
 & G-PIBT & \best{209.9\pm1.0} & \best{3.15\pm0.01} & \best{8.05\pm0.11} & \best{0.695\pm0.035} & \best{0.309\pm0.016} & $-0.338\pm0.093$ & $0.130\pm0.042$ & $0.199\pm0.011$ & $8735.6\pm2981.8$ \\
 & FAIR-RD & $436.6\pm5.9$ & $6.30\pm0.08$ & $29.27\pm1.28$ & $2.008\pm0.173$ & $1.166\pm0.076$ & $-9.385\pm0.322$ & $-1.922\pm0.165$ & $2.397\pm0.092$ & $81.3\pm29.2$ \\
 & GLC & $379.6\pm10.6$ & $4.69\pm0.09$ & $19.15\pm0.66$ & $1.554\pm0.091$ & $0.769\pm0.033$ & $0.685\pm0.204$ & $0.969\pm0.193$ & $0.426\pm0.062$ & $113.7\pm44.2$ \\
 & SPARE ($K=15,R=400$) & $309.5\pm5.5$ & $4.13\pm0.06$ & $13.66\pm0.40$ & $1.090\pm0.060$ & $0.519\pm0.025$ & \best{-0.183\pm0.121} & $0.157\pm0.084$ & \best{0.194\pm0.024} & $114.3\pm48.1$ \\
 & SPARE ($K=5,R=3000$) & $243.6\pm2.2$ & $3.55\pm0.02$ & $10.68\pm0.26$ & $0.775\pm0.048$ & $0.321\pm0.015$ & $-0.666\pm0.088$ & $-0.361\pm0.115$ & $0.234\pm0.020$ & $408.8\pm142.9$ \\
\midrule
\multirow{9}{*}{\rotatebox{90}{Chicago}} & SP & $5262.1\pm161.8$ & $54.71\pm1.18$ & $156.54\pm2.51$ & $15.937\pm0.433$ & $4.863\pm0.491$ & \best{3.412\pm3.273} & \best{1.660\pm4.147} & $10.746\pm0.326$ & \best{255.8\pm88.9} \\
 & GSP & \textsc{INC} & \textsc{INC} & \textsc{INC} & \textsc{INC} & \textsc{INC} & \textsc{INC} & \textsc{INC} & \textsc{INC} & \textsc{INC} \\
 & TAP & \textsc{OOT} & \textsc{OOT} & \textsc{OOT} & \textsc{OOT} & \textsc{OOT} & \textsc{OOT} & \textsc{OOT} & \textsc{OOT} & \textsc{OOT} \\
 & PIBT & $2390.1\pm13.6$ & $29.50\pm0.19$ & $125.09\pm1.65$ & \best{5.975\pm0.152} & $4.337\pm0.163$ & $-27.913\pm0.612$ & $-11.565\pm0.779$ & $10.233\pm0.101$ & $8149.3\pm2562.1$ \\
 & G-PIBT & \textsc{OOT} & \textsc{OOT} & \textsc{OOT} & \textsc{OOT} & \textsc{OOT} & \textsc{OOT} & \textsc{OOT} & \textsc{OOT} & \textsc{OOT} \\
 & FAIR-RD & $2579.7\pm122.9$ & $32.20\pm1.73$ & $132.33\pm5.09$ & $8.458\pm0.821$ & $6.129\pm0.510$ & $-34.079\pm4.394$ & $-9.171\pm2.335$ & $11.974\pm0.568$ & $1428.7\pm280.9$ \\
 & GLC & $2217.1\pm14.7$ & $25.97\pm0.22$ & $110.00\pm2.13$ & $6.709\pm0.208$ & $3.629\pm0.189$ & $-14.298\pm0.586$ & $-6.577\pm1.153$ & $8.380\pm0.167$ & $1615.7\pm577.0$ \\
 & SPARE ($K=15,R=400$) & $2189.6\pm13.4$ & $25.65\pm0.18$ & $109.60\pm2.03$ & $6.447\pm0.162$ & $3.556\pm0.175$ & $-13.969\pm0.505$ & $-6.464\pm1.060$ & \best{8.348\pm0.175} & $2938.1\pm1147.1$ \\
 & SPARE ($K=5,R=3000$) & \best{2178.1\pm9.3} & \best{25.53\pm0.16} & \best{109.09\pm1.47} & $6.431\pm0.232$ & \best{3.528\pm0.184} & $-13.746\pm0.604$ & $-6.429\pm1.040$ & $8.364\pm0.154$ & $14761.9\pm2790.6$ \\
\midrule
\multirow{9}{*}{\rotatebox{90}{San Francisco}} & SP & $1153.8\pm19.9$ & $12.41\pm0.23$ & $42.54\pm2.15$ & $4.708\pm0.127$ & $2.412\pm0.086$ & $4.834\pm0.256$ & $3.471\pm0.440$ & $3.081\pm0.144$ & $34.8\pm10.5$ \\
 & GSP & $1044.0\pm19.7$ & $11.53\pm0.21$ & $43.03\pm0.97$ & $4.461\pm0.090$ & $2.068\pm0.075$ & $4.055\pm0.261$ & \best{2.913\pm0.401} & $3.202\pm0.130$ & \best{28.6\pm11.0} \\
 & TAP & \textsc{OOT} & \textsc{OOT} & \textsc{OOT} & \textsc{OOT} & \textsc{OOT} & \textsc{OOT} & \textsc{OOT} & \textsc{OOT} & \textsc{OOT} \\
 & PIBT & $487.8\pm5.1$ & $7.16\pm0.07$ & $35.32\pm0.64$ & $1.855\pm0.032$ & $1.418\pm0.019$ & $-7.413\pm0.282$ & $-4.822\pm0.333$ & $2.716\pm0.065$ & $553.2\pm185.3$ \\
 & G-PIBT & \textsc{OOT} & \textsc{OOT} & \textsc{OOT} & \textsc{OOT} & \textsc{OOT} & \textsc{OOT} & \textsc{OOT} & \textsc{OOT} & \textsc{OOT} \\
 & FAIR-RD & $513.1\pm3.1$ & $7.96\pm0.05$ & $39.83\pm0.49$ & $1.868\pm0.033$ & $1.562\pm0.026$ & $-10.714\pm0.224$ & $-7.315\pm0.290$ & $3.792\pm0.057$ & $265.4\pm104.5$ \\
 & GLC & $470.1\pm5.1$ & $6.55\pm0.07$ & $28.26\pm0.48$ & $1.849\pm0.030$ & $1.309\pm0.020$ & $-2.823\pm0.294$ & $-3.221\pm0.228$ & $2.046\pm0.046$ & $230.1\pm96.9$ \\
 & SPARE ($K=15,R=400$) & $462.1\pm4.7$ & $6.46\pm0.07$ & $27.71\pm0.47$ & $1.809\pm0.025$ & $1.277\pm0.019$ & \best{-2.806\pm0.264} & $-3.219\pm0.197$ & $1.998\pm0.050$ & $402.6\pm139.1$ \\
 & SPARE ($K=5,R=3000$) & \best{449.2\pm4.0} & \best{6.31\pm0.06} & \best{27.10\pm0.46} & \best{1.754\pm0.021} & \best{1.236\pm0.020} & $-2.817\pm0.268$ & $-3.169\pm0.218$ & \best{1.939\pm0.039} & $1438.1\pm429.5$ \\
\bottomrule
\end{tabular}}
\endgroup
\end{table}
}

\section{Cross-City Explicit-Fairness Comparison}
\label{app:explicit-fairness}

Tables~\ref{tab:app-frontier-chicago} and~\ref{tab:app-frontier-sf} extend
the explicit-fairness comparison to Chicago and San Francisco. Entries use
ten fixed demand subsamples. The five ITAP settings are marked \textsc{OOT}
when their global assignment exceeds the two-hour per-run limit.

\begin{table}[H]
\centering
\normalsize
\setlength{\tabcolsep}{1.1pt}
\renewcommand{\arraystretch}{1.03}
\newcommand{\std}[1]{\mathbin{\pm}{\scriptstyle #1}}
\caption{Explicit fairness comparison on Chicago demand at $\rho=1$. Entries
are mean $\pm$ population standard deviation over ten fixed demand
subsamples. \textsc{OOT} denotes a run exceeding the two-hour per-run limit.}
\label{tab:app-frontier-chicago}
\resizebox{\columnwidth}{!}{\begin{tabular}{lrrrrr}
\toprule
\textbf{Method} & \textbf{OH (\%)} & \textbf{Mean} & \textbf{LGD} & \textbf{O-SI} & \textbf{D-SI} \\
\midrule
\multicolumn{6}{@{}l}{\textit{Assignment-based fairness methods}} \\
ITAP-$0$ & \textsc{OOT} & \textsc{OOT} & \textsc{OOT} & \textsc{OOT} & \textsc{OOT} \\
ITAP-$0.25$ & \textsc{OOT} & \textsc{OOT} & \textsc{OOT} & \textsc{OOT} & \textsc{OOT} \\
ITAP-$0.5$ & \textsc{OOT} & \textsc{OOT} & \textsc{OOT} & \textsc{OOT} & \textsc{OOT} \\
ITAP-$0.75$ & \textsc{OOT} & \textsc{OOT} & \textsc{OOT} & \textsc{OOT} & \textsc{OOT} \\
ITAP-$1$ & \textsc{OOT} & \textsc{OOT} & \textsc{OOT} & \textsc{OOT} & \textsc{OOT} \\
\midrule
\multicolumn{6}{@{}l}{\textit{Routing and coordination baselines}} \\
FAIR-RD & $1622.2 \std{8.3}$ & $20.51 \std{0.15}$ & $7.858 \std{0.113}$ & $5.054 \std{0.238}$ & $3.833 \std{0.128}$ \\
GLC & $1463.4 \std{10.9}$ & $17.31 \std{0.13}$ & $5.672 \std{0.130}$ & $4.540 \std{0.224}$ & $2.420 \std{0.104}$ \\
\midrule
\rowcolor{gray!20}
\textbf{SPARE (ours)} & $\mathbf{1436.3} \std{9.8}$ & $\mathbf{17.09} \std{0.15}$ & $5.672 \std{0.164}$ & $\mathbf{4.163} \std{0.167}$ & $\mathbf{2.351} \std{0.090}$ \\
\bottomrule
\end{tabular}
}
\end{table}

\begin{table}[H]
\centering
\normalsize
\setlength{\tabcolsep}{1.1pt}
\renewcommand{\arraystretch}{1.03}
\newcommand{\std}[1]{\mathbin{\pm}{\scriptstyle #1}}
\caption{Explicit fairness comparison on San Francisco demand at $\rho=1$.
Entries are mean $\pm$ population standard deviation over ten fixed demand
subsamples. \textsc{OOT} denotes a run exceeding the two-hour per-run limit.}
\label{tab:app-frontier-sf}
\resizebox{\columnwidth}{!}{\begin{tabular}{lrrrrr}
\toprule
\textbf{Method} & \textbf{OH (\%)} & \textbf{Mean} & \textbf{LGD} & \textbf{O-SI} & \textbf{D-SI} \\
\midrule
\multicolumn{6}{@{}l}{\textit{Assignment-based fairness methods}} \\
ITAP-$0$ & \textsc{OOT} & \textsc{OOT} & \textsc{OOT} & \textsc{OOT} & \textsc{OOT} \\
ITAP-$0.25$ & \textsc{OOT} & \textsc{OOT} & \textsc{OOT} & \textsc{OOT} & \textsc{OOT} \\
ITAP-$0.5$ & \textsc{OOT} & \textsc{OOT} & \textsc{OOT} & \textsc{OOT} & \textsc{OOT} \\
ITAP-$0.75$ & \textsc{OOT} & \textsc{OOT} & \textsc{OOT} & \textsc{OOT} & \textsc{OOT} \\
ITAP-$1$ & \textsc{OOT} & \textsc{OOT} & \textsc{OOT} & \textsc{OOT} & \textsc{OOT} \\
\midrule
\multicolumn{6}{@{}l}{\textit{Routing and coordination baselines}} \\
FAIR-RD & $341.2 \std{3.8}$ & $5.68 \std{0.06}$ & $2.656 \std{0.079}$ & $1.275 \std{0.045}$ & $1.059 \std{0.029}$ \\
GLC & $310.3 \std{3.7}$ & $4.60 \std{0.05}$ & $1.343 \std{0.040}$ & $1.233 \std{0.029}$ & $0.837 \std{0.023}$ \\
\midrule
\rowcolor{gray!20}
\textbf{SPARE (ours)} & $\mathbf{303.2} \std{3.3}$ & $\mathbf{4.51} \std{0.05}$ & $\mathbf{1.298} \std{0.037}$ & $\mathbf{1.199} \std{0.027}$ & $\mathbf{0.806} \std{0.024}$ \\
\bottomrule
\end{tabular}
}
\end{table}

\section{Cross-City Scalability}
\label{app:cross-city-scalability}

Figure~\ref{fig:density} reports the Manhattan scalability comparison, while
Figures~\ref{fig:app-density-scalability-chicago} and
\ref{fig:app-density-scalability-sf} report the same node-normalized load
sweep in Chicago and San Francisco, respectively. These two figures use the
same six metrics and method encoding as the body figure, making the city-level
differences in effectiveness, fairness, and CPU scaling directly comparable.

\begin{figure*}[t]
\centering
\includegraphics[width=0.84\textwidth]{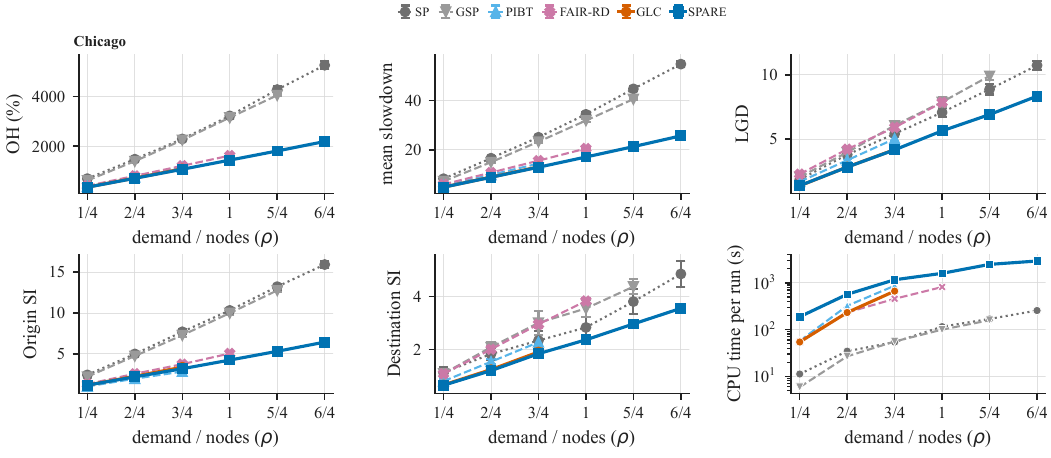}
\Description{Six scalability panels for Chicago show overhead, mean slowdown,
length-conditioned gap, origin spatial inequity, destination spatial inequity,
and per-run CPU time across six demand-to-node ratios.}
\captionof{figure}{Scalability with load ratio on Chicago demand. Lines and
error bars show mean $\pm$ population SD over ten fixed demand subsamples;
traces end when a method does not complete the next load.}
\label{fig:app-density-scalability-chicago}
\vspace{2pt}

\includegraphics[width=0.84\textwidth]{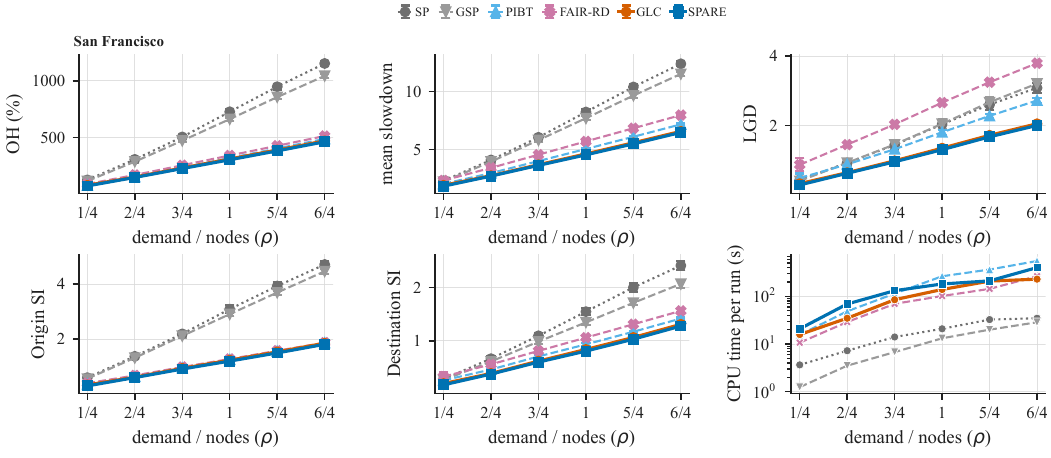}
\Description{Six scalability panels for San Francisco show overhead, mean
slowdown, length-conditioned gap, origin spatial inequity, destination spatial
inequity, and per-run CPU time across six demand-to-node ratios.}
\captionof{figure}{Scalability with load ratio on San Francisco demand. Lines
and error bars show mean $\pm$ population SD over ten fixed demand
subsamples; traces end when a method does not complete the next load.}
\label{fig:app-density-scalability-sf}
\end{figure*}

\section{Cross-City Parameter Sensitivity}
\label{app:adaptive}

Figure~\ref{fig:paramsens} reports the Manhattan sensitivity analysis, while
Figures~\ref{fig:app-paramsens-chicago} and~\ref{fig:app-paramsens-sf} show
the corresponding Chicago and San Francisco sweeps. The reference markers
identify the displayed $K$ and $R$ setting in each panel.

\begin{figure*}[t]
    \centering
    \begin{minipage}[t]{0.485\textwidth}
      \centering
      \includegraphics[width=\linewidth]{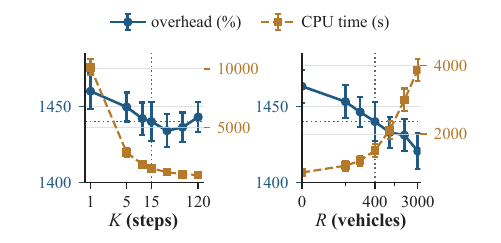}
      \Description{Two side-by-side Chicago panels show SPARE's sensitivity to review interval K with R fixed at 400 and reroute budget R with K fixed at 15. Blue curves show overhead and brown dashed curves show process CPU time. Dotted crosshairs identify the fixed reference setting, while the ordinary curve markers show its observations.}
      \captionof{figure}{Parameter sensitivity on Chicago. Curves show mean $\pm$ standard deviation over ten fixed demand subsamples.}
      \label{fig:app-paramsens-chicago}
    \end{minipage}\hfill
    \begin{minipage}[t]{0.485\textwidth}
      \centering
      \includegraphics[width=\linewidth]{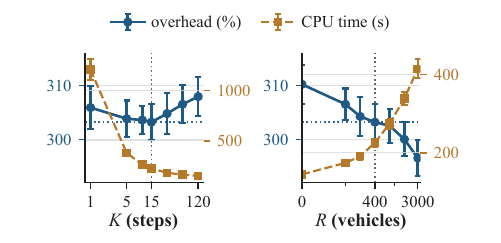}
      \Description{Two side-by-side San Francisco panels show SPARE's sensitivity to review interval K with R fixed at 400 and reroute budget R with K fixed at 15. Blue curves show overhead and brown dashed curves show process CPU time. Dotted crosshairs identify the fixed reference setting, while the ordinary curve markers show its observations.}
      \captionof{figure}{Parameter sensitivity on San Francisco. Curves show mean $\pm$ standard deviation over ten fixed demand subsamples.}
      \label{fig:app-paramsens-sf}
    \end{minipage}
\end{figure*}

\clearpage

\end{document}